\documentclass{article}

\usepackage[nonatbib,preprint]{neurips_2024}
\usepackage{caption,subcaption,wrapfig}
\usepackage{array}
\usepackage{pdfpages}
\usepackage{booktabs}

\usepackage[inline]{enumitem}
\setlist{beginpenalty=0,midpenalty=0,endpenalty=0}

\usepackage{amsmath,amsfonts,amssymb,dsfont}
\usepackage{stmaryrd}
\usepackage{nccmath}

\usepackage{xcolor}

\usepackage{tikz,tikz-cd}
\usetikzlibrary{positioning}

\tikzset{state/.style={circle, draw,
                line width=1bp,
                inner sep=0.25em,
                draw=black!85!white,
                minimum width=2em,
                fill=gray!25!white,
                }}

\tikzset{state0/.style={circle, draw,
                line width=1bp,
                inner sep=0.25em,
                draw=black!85!white,
                minimum width=2em,
                fill=white!25!white,
                }}
                
\tikzset{reward1/.style={
                draw=green!75!black,
                }}
                
\tikzset{reward05/.style={
                draw=blue!75!black,
                }}

\makeatletter
\def\bbordermatrix#1{\begingroup \m@th
  \global\let\perhaps@scriptstyle\scriptstyle
  \@tempdima 4.75\p@
  \setbox\z@\vbox{%
    \def\cr{%
      \crcr
      \noalign{%
        \kern2\p@
        \global\let\cr\endline
        \global\let\perhaps@scriptstyle\relax
      }%
    }%
    \ialign{$\make@scriptstyle{##}$\hfil\kern2\p@\kern\@tempdima
      &\thinspace\hfil$\perhaps@scriptstyle##$\hfil
      &&\quad\hfil$\perhaps@scriptstyle##$\hfil\crcr
      \omit\strut\hfil\crcr
      \noalign{\kern-\baselineskip}%
      #1\crcr\omit\strut\cr}}%
  \setbox\tw@\vbox{\unvcopy\z@\global\setbox\@ne\lastbox}%
  \setbox\tw@\hbox{\unhbox\@ne\unskip\global\setbox\@ne\lastbox}%
  \setbox\tw@\hbox{$\kern\wd\@ne\kern-\@tempdima\left[\kern-\wd\@ne
    \global\setbox\@ne\vbox{\box\@ne\kern2\p@}%
    \vcenter{\kern-\ht\@ne\unvbox\z@\kern-\baselineskip}\,\right]$}%
  \null\;\vbox{\kern\ht\@ne\box\tw@}\endgroup}
\def\make@scriptstyle#1{\vcenter{\hbox{$\scriptstyle#1$}}}
\makeatother

\makeatletter
\newcommand*{\TRANS}{%
  {\mathpalette\@transpose{}}%
}
\newcommand*{\@transpose}[2]{%
  \raisebox{\depth}{$\m@th#1\intercal$}%
}
\makeatother

\usepackage{hyperref}
\hypersetup{
    colorlinks=true,
    linkcolor=blue,
    filecolor=magenta,      
    citecolor=blue,
    urlcolor=cyan,
    }

\newcommand{\appref}[1]{\hyperref[#1]{App.~\ref*{#1}}}

\usepackage{mathtools}
\mathtoolsset{showonlyrefs=true}

\usepackage[standard,amsmath]{ntheorem}

\DeclareRobustCommand{\qed}{%
    \leavevmode\unskip
    \penalty 9999
    \hbox{}\nobreak\hfill\hbox{$\Box$}
}

\theorembodyfont{\upshape}
\newtheorem{assumption}{Assumption}

\usepackage[
backend=biber,
style=alphabetic,
maxbibnames=9,
sorting=nyt
]{biblatex}
\newcommand*\bolditalic[1]{\textbf{\textit{#1}}}

\newcommand\ALPHABET{\mathsf}
\newcommand\reals{\mathds{R}}
\newcommand\integers{\mathds{Z}}
\newcommand\naturalnumbers{\mathds{N}}

\newcommand\PR{\mathds{P}}
\newcommand\EXP{\mathds{E}}
\newcommand\IND{\mathds{1}}

\DeclareMathOperator{\SPAN}{span}
\DeclareMathOperator{\Lip}{Lip}
\DeclarePairedDelimiter{\NORM}{\lVert}{\rVert}
\DeclarePairedDelimiter\MOD{\llbracket}{\rrbracket}

\newcommand*\sublabel[1]{\textsc{\MakeLowercase{#1}}}

\newcommand\PiHND{\Pi_{\sublabel{BD}}}

\newcommand\JHND{J^\star_{\sublabel{BD}}}

\newcommand\PiZSD{\Pi_{\sublabel{SD}}}
\newcommand\PiZSS{\Pi_{\sublabel{SS}}}

\newcommand\JZSD{J^\star_{\sublabel{SD}}}

\newcommand*\EXPL{\mu}
\newcommand*\piexpl{\mu} 
\newcommand*\zexpl{\zeta_{\EXPL}}

\newcommand\MATRIX[1]{\begin{bmatrix}#1\end{bmatrix}}

\DeclareMathOperator{\diag}{blkdiag}
\DeclareMathOperator{\VEC}{vec}

\newcommand\F{\mathfrak{F}}

\usepackage{pifont}

\usepackage{authblk}

\title{Periodic agent-state based Q-learning for POMDPs}
\author[1]{Amit Sinha}
\author[2]{Matthieu Geist}
\author[1]{Aditya Mahajan}
\affil[1]{McGill University, Mila}
\affil[2]{Cohere}

\begin{document}
\maketitle

\let\originaladdcontentsline\addcontentsline
\renewcommand{\addcontentsline}[3]{}

\begin{abstract}
The standard approach for Partially Observable Markov Decision Processes (POMDPs) is to convert them to a fully observed belief-state MDP. However, the belief state depends on the system model and is therefore not viable in reinforcement learning (RL) settings. A widely used alternative is to use an agent state, which is a model-free, recursively updateable function of the observation history. Examples include frame stacking and recurrent neural networks. Since the agent state is model-free, it is used to adapt standard RL algorithms to POMDPs. However, standard RL algorithms like Q-learning learn a stationary policy. Our main thesis that we illustrate via examples is that because the agent state does not satisfy the Markov property, non-stationary agent-state based policies can outperform stationary ones. To leverage this feature, we propose PASQL (periodic agent-state based Q-learning), which is a variant of agent-state-based Q-learning that learns periodic policies. By combining ideas from periodic Markov chains and stochastic approximation, we rigorously establish that PASQL converges to a cyclic limit and characterize the approximation error of the converged periodic policy. Finally, we present a numerical experiment to highlight the salient features of PASQL and demonstrate the benefit of learning periodic policies over stationary policies.
\end{abstract}

\section{Introduction}
Recent advances in reinforcement learning (RL) have successfully integrated algorithms with strong theoretical guarantees and deep learning to achieve significant successes \cite{mnih2013playing,Silver2016MasteringTG}.  However, most RL theory is limited to models with perfect state observations~\cite{sutton2008reinforcement,neuroDP}. Despite this, there is substantial empirical evidence showing that RL algorithms perform well in partially observed settings~\cite{wierstra2007solving,Wierstra2010,hauskrecht2000value,hausknecht2015deep,Gruslys2018Reactor,kapturowski2018recurrent,Hafner2020Dream,Hafner2021mastering}. Recently, there has been a significant advances in the theoretical understanding of different RL algorithms for POMDPs~\cite{subramanian2022approximate,Kara2022,EWRL-RQL,Dong2022} but a complete understanding is still lacking.

\textbf{Planning in POMDPs.} When the system model is known, the standard approach~\cite{Astrom1965,Smallwood1973,Cassandra1994} is to construct an equivalent MDP with 
the belief state (which is the posterior distribution of the environment state given the history of observations and actions at the agent) as the information state. The belief state is policy independent and has time-homogeneous dynamics, which enables the formulation of a belief-state based dynamic program (DP). There is a rich literature which leverages the structure of the resulting DP to propose efficient algorithms to solve POMDPs~\cite{Smallwood1973,Cassandra1994,Cassandra1997,Chang2007,Zhang2009,Pineau2003,Smith2004,Spaan2005}. 
See~\cite{kochenderfer2022algorithms} for a review.
However, the belief state depends on the system model, so the belief-state based approach does not work for RL.

\textbf{RL in POMDPs.} An alternative approach for RL in POMDPs is to consider policies which depend on an \emph{agent state} $\{z_t\}_{t \ge 1}$, where $Z_t \in \ALPHABET Z$, which is a recursively updateable compression of the history: the agent starts at an initial state $z_0$ and recursively updates the agent state as some function of the current agent-state, next observation, and current action.
A simple instance of agent-state is \emph{frame stacking}, where a window of previous observations is used as state \cite{white1994finite,mnih2013playing,Kara2022}. Another example is to use a recurrent neural network such as LSTM or GRU to compress the history of observations and actions into an agent state~\cite{wierstra2007solving,Wierstra2010,hausknecht2015deep,kapturowski2018recurrent,Hafner2020Dream}. In fact, as argued in~\cite{dong2022simple,lu2023reinforcement} such an agent state is present in most deep RL algorithms for POMDPs.  We refer to such a representation as an ``agent state'' because it captures the agent's internal state that it uses for decision making. 

When the agent state is an information state, i.e., satisfies the Markov property, i.e., $\PR(z_{t+1} | z_{1:t}, a_{1:t}) = \PR(z_{t+1} | z_t, a_t)$ and is sufficient for reward prediction, i.e., $\EXP[R_t | y_{1:t}, a_{1:t}] = \EXP[R_t | z_t, a_t]$  (where $y_t$ is the observation, $a_t$ is the action, and $R_t$ is the per-step reward), the optimal agent-state based policy can be obtained via a dynamic program (DP) \cite{subramanian2022approximate}. An example of such an agent state is the belief state. But, in general, the agent state is not an information state. For example, frame stacking and RNN do not satisfy the Markov property, in general. It is also possible to have agent-states that satisfy the Markov property but are not sufficient for reward prediction (e.g., when the agent state is always a constant). In all such settings, the best agent-state policy cannot be obtained via a DP. Nonetheless, there has been considerable interest to use RL to find a good agent-state based policy.

One of the most commonly used RL algorithms is off-policy Q-learning, which we call agent-state Q-learning (ASQL). In ASQL for POMDPs, the Q-learning iteration is applied as if the agent state satisfied the Markov property even though it does not. The agent starts with an initial $Q_1(z,a)$, acts according to a behavior policy  $\piexpl$, i.e., chooses $a_t \sim \piexpl(z_t)$, and recursively updates
\begin{equation}\label{eq:ASQL}
    Q_{t+1}(z,a) =
    Q_t(z,a) + \alpha_t(z,a) \Bigl[ R_t + \gamma \max_{a' \in \ALPHABET A} Q_t(z_{t+1},a') - Q_t(z,a) \Bigr]
    \tag{ASQL}
\end{equation}
where $\gamma \in [0,1)$ is the discount factor and the learning rates $\{\alpha_t\}_{t \ge 1}$ are chosen such that $\alpha_t(z,a) = 0$ if $(z,a) \neq (z_t,a_t)$.  The convergence of~\ref{eq:ASQL} has been recently presented in~\cite{Kara2022,EWRL-RQL} which show that under some technical assumptions, \ref{eq:ASQL} converges to a limit. The policy determined by \ref{eq:ASQL} is the greedy policy w.r.t.\ this limit.

\textbf{Limitation of Q-learning with agent state.}
The greedy policy determined by \ref{eq:ASQL} is stationary (i.e., uses the same control law at every time). In infinite horizon MDPs (and, therefore, also in POMDPs when using the belief state as an agent state), stationary policies perform as well as non-stationary policies. This is because the agent-state satisfies the Markov property. However, in \ref{eq:ASQL} the agent state generally does not satisfy the Markov property. Therefore, \bolditalic{restricting attention to stationary policies may lead to a loss of optimality!}

\begin{figure}[!t]
    \centering
    \includegraphics{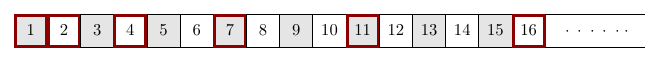}
    \caption{The cells indicate the state of the environment. Cells with the same background color have the same observation. The cells with a thick red boundary correspond to elements of the set $\ALPHABET D_0 \coloneqq \{ n(n+1)/2 + 1 : n \in \naturalnumbers \}$, where the action~$0$ gives a reward of $+1$ and moves the state to the right, while the action~$1$ gives a reward of $-1$ and resets the state to~$1$. The cells with a thin black boundary correspond to elements of the set $\ALPHABET D_1 = \naturalnumbers \setminus \ALPHABET D_0$, where the action~$1$ gives the reward of $+1$ and moves the state to the right while the action~$0$ gives a reward of $-1$ and resets the state to~$1$. Discount factor $\gamma = 0.9$.}
    \label{fig:ex1}
\end{figure}

As an illustration, consider the POMDP shown in \autoref{fig:ex1}, which is described in detail in \appref{app:ex1} as \autoref{ex:1}. Suppose the system starts in state~$1$.
Since the dynamics are deterministic,  the agent can infer the current state from the history of past actions and can take the action to increment the current state and receive a per-step reward of $+1$. Thus, the performance $\JHND$ of belief-state based policies is $\JHND = 1/(1-\gamma)$. Contrast this with the performance $\JZSD$ of deterministic agent-state base policies with agent state equal to current observation, which is given by $\JZSD = (1 + \gamma - \gamma^2)/(1 - \gamma^3) < \JHND$. In particular, for $\gamma=0.9$, $\JHND = 10$ which is larger than $\JZSD = 4.022$.

We show that the gap between $\JZSD$ and $\JHND$ can be reduced by considering non-stationary policies.
\autoref{ex:1} has deterministic dynamics, so the optimal policy can be implemented in \emph{open-loop} via a sequence of control actions $\{a^\star_t\}_{t \ge 1}$, where $a^\star_t = \IND\{t \in \ALPHABET D_1\}$. This open-loop policy can be implemented via any information structure, including agent-state based policies.
\bolditalic{Thus, a non-stationary deterministic agent-state based policy performs better than stationary deterministic agent-state based policies.} A similar conclusion also holds for models with stochastic dynamics.

\textbf{The main idea.} Arbitrary non-stationary policies cannot be used in RL because such policies have countably infinite number of parameters. In this paper, we consider a simple class of non-stationary policies with finite number of parameters: \emph{periodic policies}. An agent-state based policy $\pi = (\pi_1, \pi_2, \dots)$ is said to be periodic with period $L$ if $\pi_t = \pi_{t'}$ whenever $t \equiv t' \pmod L$. 

To highlight the salient feature of periodic policies, we perform a brute force search over all deterministic periodic policies of period $L$, for $L = \{1, \dots, 10\}$, in \autoref{ex:1}. Let $J^\star_{L}$ denote the optimal performance for policies of period~$L$. The result is shown below (see \appref{app:ex1} for details):
\[
    \begin{tabular}{@{}ccccccccccc@{}}
         \toprule
         $L$ & 1 & 2 & 3 & 4 & 5 & 6 & 7 & 8 & 9 & 10 
         \\
         \midrule
         $J^\star_L$ &
         4.022 &
         4.022 &
         7.479 &
         6.184 &
         8.810 & 
         7.479 &
         9.340 & 
         8.488 & 
         9.607 & 
         8.810 
         \\
         \bottomrule
    \end{tabular}
\]

The above example highlights some salient features of periodic policies:
\begin{enumerate*}[label=(\roman*)]
    \item Periodic deterministic agent-state based policies may outperform stationary deterministic agent-state based policies.
    \item $\{J^\star_L\}_{L \ge 1}$ is not a monotonically increasing sequence. This is because $\Pi_{L}$, the set of all periodic deterministic agent-state based policies of period~$L$, is not monotonically increasing.
    \item If $L$ divides $M$, then $J^\star_L \le J^\star_M$. This is because $\Pi_L \subseteq \Pi_M$. In other words, if we take any integer sequence $\{ L_n \}_{n \geq 1}$ that has the property that $L_n$ divides $L_{n+1}$, then the performance of the policies with period $L_n$ is monotonically increasing in $n$. For example, periodic policies with period $L \in \{ n! : n \in \naturalnumbers \}$ will have monotonically increasing performance.
    \item In the above example, the set $\ALPHABET D_0$ is chosen such that the optimal sequence of actions\footnote{Recall that the system dynamics are deterministic, so optimal policy can be implemented in open loop.} is not periodic. Therefore, even though periodic policies can achieve a performance that is arbitrarily close to the optimal belief-based policies, they are not necessarily globally optimal (in the class of non-stationary agent-state based policies).
    Thus, the periodic deterministic policy class is a middle ground between the stationary deterministic and non-stationary policy classes and provides us a simple way of leveraging the benefits of non-stationarity while trading-off computational and memory complexity.
\end{enumerate*}

The main contributions of this paper are as follows.
\begin{enumerate}[itemsep=0pt,topsep=0pt,partopsep=0pt]
    \item Motivated by the fact that non-stationary agent-state based policies outperform stationary ones, we propose a variant of agent-state based Q-learning (ASQL) that learns periodic policies. We call this algorithm periodic agent-state based Q-learning (PASQL). 
    \item We rigorously establish that PASQL converges to a cyclic limit. Therefore, the greedy policy w.r.t.\ the limit is a periodic policy. Due to the non-Markovian nature of the agent-state, the limit (of the Q-function and the greedy policy) depends on the behavioral policy used during learning. 
    \item We quantify the sub-optimality gap of the periodic policy learnt by PASQL. 
    \item We present numerical experiments to illustrate the convergence results, highlight the salient features of PASQL, and show that the periodic policy learned by PASQL indeed performs better than stationary policies learned by ASQL. 
\end{enumerate}

\section{Periodic agent-state based Q-learning (PASQL) with agent state}

\subsection{Model for POMDPs}
A POMDP is a stochastic dynamical system with state $s_t \in \ALPHABET S$, input $a_t \in \ALPHABET A$, and output $y_t \in \ALPHABET Y$, where we assume that all sets are finite. The system operates in discrete time with the dynamics given as follows: The initial state $s_1 \sim \rho$ and  for any time $t \in \naturalnumbers$, we have
\begin{equation}
    \PR(s_{t+1}, y_{t+1} \mid s_{1:t}, y_{1:t}, a_{1:t})
    =
    \PR(s_{t+1}, y_{t+1} \mid s_{t}, a_{t})
    \eqqcolon P(s_{t+1}, y_{t+1} \mid s_t, a_t)
\end{equation}
where $P$ is a probability transition matrix. In addition, at each time the system yields a reward $R_t = r(s_t, a_t)$. We will assume that $R_t \in [0, R_{\max}]$. The discount factor is denoted by $\gamma \in [0,1)$. 

Let $\vec {\pi} = (\vec \pi_1, \vec \pi_2, \dots)$ denote any (history dependent and possibly randomized) policy. Then the action at time~$t$ is given by $a_t \sim \vec \pi_t(y_{1:t}, a_{1:t-1})$. The performance of policy $\vec {{\pi}}$ is given by 
\begin{equation}
    J^{\vec{{\pi}}} \coloneqq \EXP_{\substack{a_t \sim \vec\pi_t(y_{1:t}, a_{t-1}) \\ (s_{t+1}, y_{t+1}) \sim P(s_t, a_t)}}
    \biggl[ \medop\sum_{t=1}^\infty \gamma^{t-1} r(s_t, a_t)  \biggm| s_1 \sim \rho \biggr].
\end{equation}
The objective is to find a (history dependent and possibly randomized) policy $\vec{{\pi}}$ to maximize $J^{\vec{{\pi}}}$. 

\textbf{Agent state for POMDPs.}
An agent-state is model-free recursively updateable function of the history of observations and actions. In particular, let $\ALPHABET Z$ denote agent-state space. Then, the agent state process $\{z_t\}_{t \ge 0}$, $z_t \in \ALPHABET Z$,  starts with an initial value $z_0$, and is then recursively computed as
\(
    z_{t+1} = \phi(z_t, y_{t+1}, a_t)
\)    
for a pre-specified agent-state update function $\phi$. 

We use ${\pi} = (\pi_1, \pi_2, \dots)$ to denote an agent-state based policy,\footnote{We use $\vec{\pi}$ to denote history dependent policies and $\pi$ to denote agent-state based policies.} i.e., a policy where the action at time~$t$ is given by $a_t \sim \pi_t(z_t)$. 
An agent-state based policy is said to be \textbf{stationary} if for all $t$ and $t'$, we have $\pi_t(a|z) = \pi_{t'}(a|z)$ for all $(z,a) \in \ALPHABET Z \times \ALPHABET A$. 
An agent-state based policy is said to be \textbf{periodic} with period $L$ if for all $t$ and $t'$ such that $t \equiv t' \pmod L$, we have $\pi_t(a|z) = \pi_{t'}(a|z)$ for all $(z,a) \in \ALPHABET Z \times \ALPHABET A$. 

\subsection{PASQL: Periodic agent-state based Q-learning algorithm for POMDPs}

We now present a periodic variant of agent-state based Q-learning, which we abbreviate as PASQL. PASQL is an online, off-policy learning algorithm in which the agent acts according to a behavior policy $\piexpl = (\piexpl_1, \piexpl_2, \dots)$ which is a periodic (stochastic) agent-state based policy $\mu$ with period $L$.

Let $\MOD{t} \coloneqq (t \bmod L)$ and $\ALPHABET L \coloneqq \{0, 1, \dots, L-1\}$. Let $(z_1, a_1, R_1, z_2, a_2, R_2, \dots)$ be a sample path of agent-state, action, and reward observed by the agent. In PASQL, the agent maintains an $L$-tuple of Q-functions $(Q^0_{t}, Q^1_{t}, \dots, Q^{L-1}_{t})$, $t \ge 1$. The $\ell$-th component, $\ell \in \ALPHABET L$, is updated at time steps when $\MOD{t} = \ell$. In particular, we can write the update as
\begin{equation}
    Q^\ell_{t+1}(z,a) = Q^\ell_{t}(z,a) + \alpha^{\ell}_t(z,a) \Bigl[
       R_t + \gamma \max_{a' \in \ALPHABET A} Q^{\MOD{\ell+1}}_t(z_{t+1}, a') - Q^{\ell}_t(z,a) 
    \Bigr],
    \quad \forall \ell \in \ALPHABET L,
    \tag{PASQL}
    \label{eq:PASQL}
\end{equation}  
where the learning rate sequence $\{(\alpha^0_t, \dots, \alpha^{L-1}_t)\}_{t \ge 1}$ is chosen such that $\alpha^{\ell}_t(z,a) = 0$ if $(\ell, z,a) \neq (\MOD{t}, z_t,a_t)$ and satisfies \autoref{ass:lr}. \ref{eq:PASQL} differs from~\ref{eq:ASQL} in two aspects: 
\begin{enumerate*}[label=(\roman*)]
    \item The behavior policy $\piexpl$ is periodic.
    \item The update of the Q-function is periodic.
\end{enumerate*}
When $L = 1$, \ref{eq:PASQL} collapses to \ref{eq:ASQL}.

The standard convergence analysis of Q-learning for MDPs shows that the Q-function convergences to the unique solution of the MDP dynamic program (DP). The key challenge in characterizing the convergence of~\ref{eq:PASQL} is that the agent state $\{Z_t\}_{t \ge 1}$ does not satisfy the Markov property. Therefore, a DP to find the best agent-state based policy does not exist. So, we cannot use the standard analysis to characterize the convergence of \ref{eq:PASQL}. In \autoref{sec:convergence}, we provide a complete characterization of the convergence of \ref{eq:PASQL}.

The quality of the converged solution depends on the expressiveness of the agent state. For example, if the agent state is not expressive (e.g., agent state is always constant), then even if \ref{eq:PASQL} converges to a limit, the limit will be far from optimal. Therefore, it is important to quantify the degree of sub-optimality of the converged limit. We do so in \autoref{sec:AIS}.
\subsection{Establishing the convergence of tabular PASQL} \label{sec:convergence}

To characterize the convergence of tabular PASQL, we impose two assumptions which are standard for analysis of RL algorithms~\cite{Jaakkola1994,neuroDP}. The first assumption is on the learning rates.
\begin{assumption}\label{ass:lr}
    For all $(\ell,z,a)$, 
    the learning rates $\{\alpha^{\ell}_t(z,a)\}_{t \ge 1}$ are measurable with respect to the sigma-algebra generated by $(z_{1:t}, a_{1:t})$ and satisfy $\alpha^{\ell}_t(z,a) = 0$ if $(\ell, z,a) \neq (\MOD{t}, z_t,a_t)$. Moreover, 
    \(
        \sum_{t \ge 1} \alpha^{\ell}_t(z,a) = \infty
    \)
    and
    \(
        \sum_{t \ge 1} (\alpha^{\ell}_t(z,a))^2 < \infty
    \), almost surely.
\end{assumption}

The second assumption is on the behavior policy~$\piexpl$. We first state an immediate property.
\begin{lemma}\label{lem:Markov}
    For any behavior policy $\piexpl$, the process $\{(S_t, Z_t)\}_{t \ge 1}$ is Markov. Therefore, the processes $\{(S_t, Z_t, A_t)\}_{t \ge 1}$ and $\{(S_t, Y_t, Z_t, A_t)\}_{t \ge 1}$ are also Markov.
\end{lemma}
\begin{assumption}\label{ass:policy}
    The behavior policy $\piexpl$ is such that the Markov chain $\{(S_t, Y_t, Z_t, A_t)\}_{t \ge 1}$ is time-periodic\footnote{Time-periodic Markov chains are a generalization of time-homogeneous Markov chains. We refer the reader to \appref{app:periodic-MC} for an overview of time-periodic Markov chains and cyclic limiting distributions, including sufficient conditions for the existence of such distributions.}
    with period~$L$ and converges to a cyclic limiting distribution $(\zexpl^0, \dots, \zexpl^{L-1})$, where $\sum_{(s,y)} \zexpl^{\ell}(s,y,z,a) > 0$ for all $(\ell,z,a)$ (i.e., all $(\ell,z,a)$ are visited infinitely often).
\end{assumption}

For the ease of notation, we will continue to use $\zexpl^\ell$ to denote the marginal and conditional distributions w.r.t.\ $\zexpl^\ell$. In particular, for marginals we use $\zexpl^\ell(y,z,a)$ to denote $\sum_{s \in \ALPHABET S} \zexpl^\ell(s,y,z,a)$ and so on; for conditionals, we use  $\zexpl^\ell(s | z,a)$ to denote $\zexpl^\ell(s,z,a) / \zexpl^\ell(z,a)$ and so on. Note that $\zexpl^\ell(s,z,y,a) = \zexpl^\ell(s,z) \EXPL (a \mid z) P(y|s,a)$. Thus, we have that $\zexpl^\ell(s \mid z, a) = \zexpl^\ell(s \mid z)$.

\begin{theorem}\label{thm:convergence}
    Under \hyperref[ass:lr]{Assms.\@~\ref{ass:lr}} and~\ref{ass:policy}, the process $\{(Q^0_t, \dots, Q^{L-1}_t)\}_{t \ge 1}$ converges to a limit $(Q^0_\EXPL, \dots, Q^{L-1}_{\EXPL})$ a.s., where the limit is the unique fixed point of the DP for a periodic MDP:\footnote{See \appref{app:periodic-MDP} for an overview of periodic MDPs.}
    \begin{equation}\label{eq:PDP}
        Q^\ell_{\EXPL}(z,a) = r^\ell_{\EXPL}(z,a) + \gamma 
        \medop\sum_{z' \in \ALPHABET Z} P^\ell_{\EXPL}(z' | z,a) \max_{a' \in \ALPHABET A} Q^{\MOD{\ell+1}}_{\EXPL}(z',a'),
        \quad \forall \ell \in \ALPHABET L, 
        \forall (z,a) \in \ALPHABET Z \times \ALPHABET A
    \end{equation}
    where the periodic rewards $(r^0_{\EXPL}, \dots, r^{L-1}_{\EXPL})$ and dynamics $(P^0_{\EXPL}, \dots, P^{L-1}_{\EXPL})$ are given by
    \begin{align}\label{eq:periodic-MDP}
        r^\ell_{\EXPL}(z,a) &\coloneqq \medop\sum_{s \in \ALPHABET S} r(s,a) \zexpl^\ell(s \mid z),
        &
        P^\ell_{\EXPL}(z'|z,a) &\coloneqq \smashoperator{{\medop\sum}_{ (s,y') \in \ALPHABET S \times \ALPHABET Y}}
        \IND_{\{z' = \phi(z,y',a)\}} P(y'|s,a) \zexpl^\ell(s|z).  
    \end{align}
\end{theorem}
See \appref{app:PASQL} for proof. Some salient features of the result are as follows:

\begin{itemize}[itemsep=0pt,topsep=0pt, partopsep=0pt]
    \item In contrast to Q-learning for MDPs, the limiting value $Q^\ell_{\mu}$ depends on the behavioral policy $\mu$. This dependence arises because the agent state $Z_t$ is not an information state and thus is not policy-independent. See \cite{witsenhausen1975policy} for a discussion on policy independence of information states.
    \item  We can recover some existing results in the literature as special cases of \autoref{thm:convergence}. If we take $L=1$, \autoref{thm:convergence} recovers the convergence result for \ref{eq:ASQL} obtained in \cite[Thm.~2]{EWRL-RQL}. In addition, if the agent state is a sliding window memory, \autoref{thm:convergence} recovers the convergence result obtained in \cite[Thm.~4.1]{Kara2022}. 
    Note that the results of \autoref{thm:convergence} for these special cases is more general because the
    previous results were derived under a restrictive assumption on the learning rates.  
\end{itemize}

The policy learned by \ref{eq:PASQL} is the periodic policy  $\pi_{\EXPL} = (\pi_{\EXPL}^0, \dots, \pi_{\EXPL}^{L-1})$  given by
\begin{equation}\label{eq:policy}
    \pi^\ell_{\EXPL}(z) = \arg \max_{a \in \ALPHABET A}Q^\ell_{\EXPL}(z,a),
    \quad \forall \ell \in \ALPHABET L, z \in \ALPHABET Z.
    \tag{PASQL-policy}
\end{equation}
Since \ref{eq:PASQL} learns a periodic policy, it circumvents the limitation of \ref{eq:ASQL} described in the introduction. \autoref{thm:convergence} addresses the main challenge in the convergence analysis of \ref{eq:PASQL}: the non-Markovian dynamics of $\{Z_t\}_{t \ge 1}$. A natural follow-up question is: How good is the learnt policy~\eqref{eq:policy} compared to the optimal? We address this in the next section.

\subsection{Characterizing the optimality-gap of the converged limit} \label{sec:AIS}

\textbf{History-dependent policies and their value functions.} 
Let $h_t = (y_{1:t}, a_{1:t-1})$ denote the history of observations and actions until time~$t$.
and let $\sigma_t \colon h_t \mapsto z_t$ denotes the map from histories to agent-states obtained by unrolling the memory update function~$\phi$, i.e., 
    $\sigma_1(h_1) = \phi(z_0, y_1, a_0)$, where $z_0$ is the initial agent state, $a_0$ is a dummy action used to initialize the process,
    $\sigma_2(h_2) = \phi(\sigma_1(h_1),y_2,a_1)$, 
    etc.

For any history dependent policy
\(
    \vec{ \pi} = (\vec \pi_1,  \vec \pi_2, \cdots)
\),
where $\vec \pi_t \colon h_t \mapsto a_t$, let 
\(
    V^{\vec \pi}_t(h_t) \coloneqq \EXP^{\vec \pi}\bigl[ \sum_{\tau = t}^{\infty} \gamma^\tau R_{\tau} \bigm| h_t \Bigr]
\)
denote the value function of policy $\vec \pi$ starting from history~$h_t$ at time~$t$. Let $V^\star_t(h_t) \coloneqq \sup_{\vec \pi} V^{\vec \pi}_t(h_t)$ denote the optimal value function, where the supermum is over all history dependent policies. In \autoref{thm:convergence}, we have shown that \ref{eq:PASQL} converges to a limit. Let $\vec \pi_{\mu} = (\vec \pi_{\mu,1}, \vec \pi_{\mu,2}, \dots)$ denote the history dependent policy corresponding to the periodic policy $(\pi^0_\mu, \dots, \pi^{L-1}_{\mu})$ given by~\eqref{eq:policy}, i.e.,
\(
    \vec \pi_{\mu,t}(h_t) \coloneqq \pi^{\MOD{t}}(\sigma_t(h_t)),
\)
In this section, we present a bound on the sub-optimality gap $V^\star_t(h_t) - V^{\vec \pi_{\mu}}_t(h_t)$. 

\textbf{Integral probability metrics.} Let $\F$ be a convex and balanced\footnote{$\F$ is balanced means that for every $f \in \F$ and scalar $a$ such that $|a|\le 1$, we have $af \in \F$.} subset of (measureable) real-valued functions on $\ALPHABET S$. The integral probability metric (IPM) w.r.t.\ $\F$, denoted by $d_{\F}$, is defined as follows:  any probability measures $\xi_1$ and $\xi_2$ on $\ALPHABET S$, we have
\(
    d_{\F} (\xi_1, \xi_2) \coloneqq \sup_{f \in \F} \bigl\lvert \int f d\xi_1 - \int f d\xi_2 \bigr\rvert.
\)
Moreover, for any real-valued function $f$ on $\ALPHABET S$, define $\rho_{\F} \coloneqq \inf\{ \rho > 0 \colon f / \rho \in \F\}$ to be the Minkowski functional w.r.t.\ $\F$. 
Note that if for every positive $\rho$, $f/\rho \not\in \F$, then $\rho_{\F}(f) = \infty$.

Many commonly used metrics on probability spaces are IPMs. For example,
\begin{enumerate*}[label=(\roman*)]
\item Total variation distance for which $\F = \{ \SPAN(f) \le 1 \}$, where $\SPAN(f) = \max f - \min f$ is the span seminorm of~$f$. In this case, $\rho_{\F}(f) = \SPAN(f)$. 
\item Wasserstein distance for which $\F = \{ \Lip(f) \le 1 \}$, where $\Lip(f)$ is the Lipschitz constant of $f$. In this case, $\rho_{\F}(f) = \Lip(f)$.  
\end{enumerate*}
Other examples include Kantorovich metric, bounded Lipschitz metric, and maximum mean discrepancy. See~\cite{muller1997integral,subramanian2022approximate} for more details.

\textbf{Sub-optimality gap.}
Let $\ALPHABET T(t,\ell) \coloneqq \{ \tau \ge t : \MOD{\tau} = \ell \}$.
Furthermore, for any $\ell \in \ALPHABET L$ and~$t$, define
\begin{align} \label{eq:pasql-approx-eps-delta}
    \varepsilon^{\ell}_{t} &\coloneqq \sup_{\tau \in \ALPHABET T(t,\ell)} \sup_{h_{\tau}, a_{\tau}}
    \Bigl\lvert \EXP [R_{\tau} \mid h_{\tau}, a_{\tau}] - \sum_{s \in \ALPHABET S}  r(s,a_{\tau}) \zexpl^\ell(s \mid \sigma_\tau(h_\tau),a_\tau)
\Bigr\rvert, \\
\delta^{\ell}_t &\coloneqq \sup_{\tau \in \ALPHABET T(t,\ell)} \sup_{h_{\tau}, a_{\tau}} d_{\F} (\PR(Z_{\tau+1} = \cdot \mid h_{\tau}, a_{\tau}), P^\ell_{\EXPL}(Z_{\tau+1} = \cdot |\sigma_{\tau}(h_{\tau}),a_{\tau})).
\end{align}
Then, we have the following sub-optimality gap for $\vec \pi_{\mu}$. 
\begin{theorem}\label{thm:approx}
Let $V^{\ell}_{\EXPL} (z) \coloneqq \max_{a \in \ALPHABET A} Q^{\ell}_{\EXPL} (z, a)$.  Then,
\begin{equation}\label{eq:approx-bound}
    \sup_{h_t} \bigl[ V^{\star}_t (h_t) - V^{\vec \pi_{\mu}}_t(h_t) \bigr]
    \leq \frac{2}{(1-\gamma^L)} \sum_{\ell \in \ALPHABET L} \gamma^{\ell} \Bigl[\varepsilon^{\MOD{t+\ell}}_{t + \ell} + \gamma \delta^{\MOD{t+\ell}}_{t + \ell} \rho_{\F} (V^{\MOD{t+\ell+1}}_{\EXPL})
    \Bigr].
\end{equation}
\end{theorem}
See \appref{app:approx} for proof.
The salient features of the sub-optimality gap of \autoref{thm:approx} are as follows. 
\begin{itemize}[itemsep=0pt,topsep=0pt, partopsep=0pt]
    \item  We can recover some existing results as special cases of \autoref{thm:approx}. When we take $L=1$, \autoref{thm:approx} recovers the sub-optimality gap for \ref{eq:ASQL} obtained in \cite[Thm.~3]{EWRL-RQL}. In addition, when the agent state is a sliding window memory, \autoref{thm:approx} is similar to the sub-optimality gap obtained in \cite[Thm.~4.1]{Kara2022}. 
    Note that the results of \autoref{thm:approx} for these special cases is more general because the
    previous results were derived under a restrictive assumption on the learning rates.  
    \item The sub-optimality gap in \autoref{thm:approx} is on the sub-optimality w.r.t.\ the optimal \emph{history-dependent} policy rather than the optimal non-stationary agent-state policy. Thus, it inherently depends on the quality of the agent state. Consequently, even if $L \to \infty$, the sub-optimality gap does not go to zero. 
    
    \item 
    It is not easy to characterize the sensitivity of the bound to the period $L$. In particular, increasing $L$ means changing behavioral policy $\mu$, and therefore changing the converged limit $(\zeta^0_{\mu}, \dots, \zeta^{L-1}_{\mu})$, which impacts the right hand side of~\eqref{eq:approx-bound} in a complicated way. So, it is not necessarily the case that increasing $L$ reduces the sub-optimality gap. This is not surprising, as we have seen earlier in \autoref{ex:1} presented in the introduction that even the performance of periodic agent-state based policies is not monotone in~$L$.
\end{itemize}

\section{Numerical experiments}\label{sec:experiments}

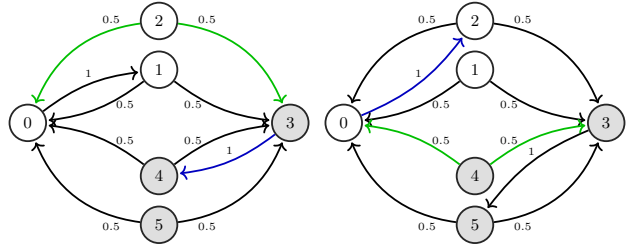
\begin{wrapfigure}{r}{8.2cm}
    \vskip -4\baselineskip
    \centering
    \begin{subfigure}{4cm}
    \centering
    \resizebox{4cm}{!}{\begin{tikzpicture}[line width=1pt, node distance=0.5cm and 2cm,font=\small]
        \node[state0] (S1) {$0$};
        \node[state0, above right=of S1] (S2) {$1$};
        \node[state0, above=2mm of S2] (S3) {$2$};
        
        \node[state, below right=of S2] (S4) {$3$};
        
        \node[state, below right=of S1] (S5) {$4$};
        \node[state, below=2mm of S5] (S6) {$5$};

        \draw[every loop]
              (S1) edge[bend left=15, auto=left, pos=0.5] node[above] {\tiny $1$} (S2)
              (S2) edge[bend left=15, auto=left, pos=0.2]  node[below] {\tiny $0.5$} (S1)
              (S2) edge[bend right=15, auto=right, pos=0.2]  node[below] {\tiny $0.5$} (S4)
              (S3) edge[reward1,bend right, auto=right, pos=0.2]  node[above] {\tiny $0.5$} (S1)
              (S3) edge[reward1,bend left, auto=left, pos=0.2]  node[above] {\tiny $0.5$} (S4)
              (S4) edge[reward05,bend left=15, auto=left] node[above] {\tiny $1$} (S5)
              (S5) edge[bend left=15, auto=left, pos=0.2]  node[above] {\tiny $0.5$} (S4)
              (S5) edge[bend right=15, auto=right, pos=0.2]  node[above] {\tiny $0.5$} (S1)
              (S6) edge[bend right, auto=right, pos=0.2]  node[below] {\tiny $0.5$} (S4)
              (S6) edge[bend left, auto=left, pos=0.2]  node[below] {\tiny $0.5$} (S1)
            ;

    \end{tikzpicture}}
    \caption{Dynamics under action~$0$.}
    \end{subfigure}
    \hfill
    \begin{subfigure}{4cm}
    \centering
    \resizebox{4cm}{!}{\begin{tikzpicture}[line width=1pt, node distance=0.5cm and 2cm,font=\small]
        \node[state0] (S1) {$0$};
        \node[state0, above right=of S1] (S2) {$1$};
        \node[state0, above=2mm of S2] (S3) {$2$};
        
        \node[state, below right=of S2] (S4) {$3$};
        
        \node[state, below right=of S1] (S5) {$4$};
        \node[state, below=2mm of S5] (S6) {$5$};

        \draw[every loop]
              (S1) edge[reward05,bend right=15, auto=left, pos=0.5] node[above] {\tiny $1$} (S3)
              (S2) edge[bend left=15, auto=left, pos=0.2]  node[below] {\tiny $0.5$} (S1)
              (S2) edge[bend right=15, auto=right, pos=0.2]  node[below] {\tiny $0.5$} (S4)
              (S3) edge[bend right, auto=right, pos=0.2]  node[above] {\tiny $0.5$} (S1)
              (S3) edge[bend left, auto=left, pos=0.2]  node[above] {\tiny $0.5$} (S4)
              (S4) edge[bend right=15, auto=left, pos=0.5] node[above] {\tiny $1$} (S6)
              (S5) edge[reward1, bend left=15, auto=left, pos=0.2]  node[above] {\tiny $0.5$} (S4)
              (S5) edge[reward1, bend right=15, auto=right, pos=0.2]  node[above] {\tiny $0.5$} (S1)
              (S6) edge[bend right, auto=right, pos=0.2]  node[below] {\tiny $0.5$} (S4)
              (S6) edge[bend left, auto=left, pos=0.2]  node[below] {\tiny $0.5$} (S1)
            ;

    \end{tikzpicture}}
    \caption{Dynamics under action~$1$.}
    \end{subfigure}
    \vskip -0.25\baselineskip
    \caption{The model for \autoref{ex:PASQL-example}, where states which have the same color give the same observation; the green edges give a reward of $+1$ and blue edges give a reward of $+0.5$.}
    \label{fig:PASQL-example}
    \vskip -1.5\baselineskip
\end{wrapfigure}

In this section, we present a numerical example to highlight the salient features of our results. We use the following POMDP model.

\begin{example}\label{ex:PASQL-example}
    Consider a POMDP with $\ALPHABET S = \{0, 1, \dots, 5\}$, $\ALPHABET A = \{0, 1\}$, $\ALPHABET Y = \{0, 1\}$ and $\gamma=0.9$. The dynamics are as shown in \autoref{fig:PASQL-example}. The observation is $0$ in states $\{0,1,2\}$ which are shaded white and is $1$ in states $\{3,4,5\}$ which are shaded gray. The transitions shown in green give a reward of $+1$; those in in blue give a reward of $+0.5$; others give no reward. 
\end{example}

We consider a family of models, denoted by $\mathcal M(p)$, $p \in [0,1]$, which are similar to \autoref{ex:PASQL-example} except the controlled state transition matrix is $p I + (1-p) P$, where $P$ is the controlled state transition matrix of \autoref{ex:PASQL-example} shown in \autoref{fig:PASQL-example}. In the results reported below, we use $p = 0.01$. The hyperparameters for the experiments are provided in \appref{app:reproducibility}.

\textbf{Convergence of \ref{eq:PASQL} with $L=2$.}
We assume that the agent state $Z_t = Y_t$ and take period $L=2$. We consider three behavioral policies: $\mu_k = (\mu_k^0, \mu_k^1)$, $k \in \ALPHABET K \coloneqq \{1,2,3\}$, where $\mu_k^\ell \colon \{0, 1\} \to \Delta(\{0,1\})$, $\ell \in \{0,1\}$. The policy $\mu_k$ is completely characterized by four numbers which we write in matrix form as: $[\mu_k^0(0|0), \mu_k^1(0|0); \mu_k^0(0|1), \mu_k^1(0|1)]$. With this notation, the three policies are given by
\(
    \mu_{1} \coloneqq
    [0.2 , 0.8 ; 0.8 , 0.2]
\)
    ,
\(
    \mu_2 \coloneqq 
    [0.5 , 0.5 ; 0.5 , 0.5]
\)
    ,
\(
    \mu_3 \coloneqq
    [0.8 , 0.2 ; 0.2 , 0.8]
\)
    .

\begin{figure}[!t]
    \centering
    \includegraphics[width=0.95\textwidth]{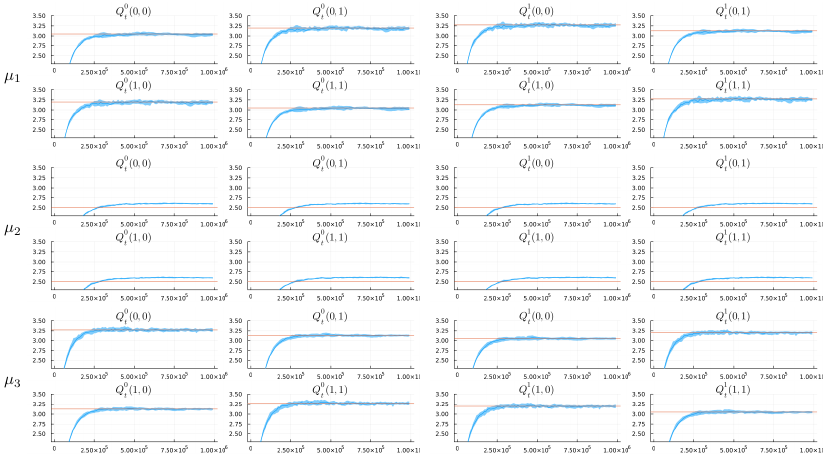}
    \caption{\ref{eq:PASQL} iterates for different behavioral policies (in blue) and the limit predicted by \autoref{thm:convergence} (in red).}
    \label{fig:plots-period2}
\end{figure}

For each behavioral policy $\mu_k$, $k \in \ALPHABET K$,  run \ref{eq:PASQL} for 25 random seeds. The median + interquantile range of the iterates $\{Q^\ell_t(z,a)\}_{t \ge 1}$ as well as the theoretical limits $Q_{\mu_k}(z,a)$ (computed using \autoref{thm:convergence}) are shown in \autoref{fig:plots-period2}.
The salient features of these results are as follows:
\begin{itemize}[noitemsep,topsep=0pt, partopsep=0pt]
    \item \ref{eq:PASQL} converges close to the theoretical limit predicted by \autoref{thm:convergence}.
    \item As highlighted earlier, the limiting value $Q^\ell_{\mu_k}$ depends on the behavioral policy $\mu_k$. 
    \item When the aperiodic behavior policy $\mu_2$ is used, the Markov chain $\{(S_t,Y_t,Z_t,A_t)\}_{t \ge 1}$ is aperiodic, and therefore the limiting distribution $\zeta^{\ell}_{\mu_2}$ and the corresponding Q-functions $Q^\ell_{\mu_2}$ do not depend on~$\ell$. This highlights the fact that we have to choose a periodic behavioral policy to converge to a non-stationary policy~\eqref{eq:policy}.
\end{itemize}

\begin{wraptable}{r}{4.5cm}
    \centering
   \caption{Performance of converged periodic policies.}
   \label{tab:performance-period2}
   \begin{tabular}{@{}cccc@{}}
        \toprule
        $J^{\star}_2$ &
        $J^{\pi_{\mu_1}}$  &
        $J^{\pi_{\mu_2}}$  &
        $J^{\pi_{\mu_3}}$  
        \\
        6.793 &
        6.793 &
        1.064 &
        0.532 \\
        \bottomrule
   \end{tabular} 
    \vskip -\baselineskip
\end{wraptable}
\textbf{Comparison of converged policies.}
Finally, we compute the periodic greedy policy $\pi_{\mu_k} = (\pi_{\mu_k}^0, \pi_{\mu_k}^1)$ given by~\eqref{eq:policy}, $k \in \ALPHABET K$, and compute its performance $J^{\pi_{\mu_k}}$ via policy evaluation on the product space $\ALPHABET S \times \ALPHABET Z$ (see \appref{app:evaluation}). 
We also do a brute force search over all $L=2$ periodic deterministic agent-state policies to compute the optimal performance $J^\star_2$ over all such policies.  The results, displayed in \autoref{tab:performance-period2}, illustrate the following:
\begin{itemize}[noitemsep,topsep=0pt, partopsep=0pt]
    \item The greedy policy $\pi_{\mu_k}$ depends on the behavioral policy. This is not surprising given the fact that the limiting value $Q^\ell_{\mu_k}$ depends on $\mu_k$. 
    \item  The policy $\pi_{\mu_1}$ achieves the optimal performance, whereas the policies $\pi_{\mu_2}$ and $\pi_{\mu_3}$ do not perform well. This highlights the importance of starting with a good behavioral policy. See \autoref{sec:discussion} for a discussion on variants such as $\epsilon$-greedy.
\end{itemize}

\begin{wraptable}{R}{4.5cm}
    \vskip -\baselineskip
    \centering
   \caption{Performance of converged stationary policies.}
   \label{tab:performance-period1}
   \begin{tabular}{@{}cccc@{}}
        \toprule
        $J^{\star}_1$ &
        $J^{\pi_{\bar \mu_1}}$  &
        $J^{\pi_{\bar \mu_2}}$  &
        $J^{\pi_{\bar \mu_3}}$  
        \\
        2.633 &
        0.0 &
        1.064 &
        2.633 \\
        \bottomrule
   \end{tabular} 
    \vskip -\baselineskip
\end{wraptable}
\textbf{Advantage of learning periodic policies.}
As stated in the introduction, the main motivation of \ref{eq:PASQL} is that it allows us to learn non-stationary policies. To see why this is useful, we run \ref{eq:ASQL} (which is effectively \ref{eq:PASQL} with $L=1$). We again consider three behavioral policies: $\bar \mu_k$, $k \in \ALPHABET K \coloneqq \{1,2,3\}$, where $\bar \mu_k \colon \{0, 1\} \to \Delta(\{0,1\})$, where (using similar notation as for $L=2$ case)

\(
    \bar \mu_{1} \coloneqq
    [0.2; 0.8]
\)
    ,
\(
    \bar \mu_2 \coloneqq 
    [0.5; 0.5]
\)
    ,
\(
    \bar \mu_3 \coloneqq
    [0.8; 0.2]
\)
    .

For each behavioral policy $\bar \mu_k$, $k \in \ALPHABET K$, run \ref{eq:ASQL} for 25 random seeds. The results are shown in \appref{app:ex:PASQL-example}. The performance of the greedy policies $\pi_{\bar \mu_k}$ and the performance of the best period $L=1$ deterministic agent-state-based policy computed via brute force is shown in \autoref{tab:performance-period1}. The key implications are as follows:
\begin{itemize}[noitemsep,topsep=0pt, partopsep=0pt]
    \item As was the case for \ref{eq:PASQL}, the greedy policy $\pi_{\bar \mu_k}$ depends on the behavioral policy. As mentioned earlier, this is a fundamental consequence of the fact that the agent state is not an information state. Adding (or removing) periodicity does not change this feature. 
    \item The best performance of \ref{eq:ASQL} is worse than the best performance of \ref{eq:PASQL}. This highlights the potential benefits of using periodicity. However, at the same time, if a bad behavioral policy is chosen (e.g., policy $\mu_3$), the performance of \ref{eq:PASQL} can be worse than that of \ref{eq:ASQL} for a nominal policy (e.g., policy $\bar \mu_2$). This highlights that periodicity is not a magic bullet and some care is needed to choose a good behavioral policy. Understanding what makes a good periodic behavioral policy is an unexplored area that needs investigation.
\end{itemize}

\section{Related work} \label{sec:related-work}

\textbf{Policy search for agent state policies.} There is a rich literature on planning with agent state-based policies that build on the policy evaluation formula presented in \appref{app:evaluation}. See \cite{kochenderfer2022algorithms} for review. These approaches rely on the system model and cannot be used in the RL setting.

\textbf{State abstractions for POMDPs} are related to agent-state based policies.
Some frameworks for state abstractions in POMDPs include predictive state representations (PSR)~\cite{RosencrantzGordonThrun_2004,Boots2011,Hamilton2014,Kulesza2015,Kulesza2015a,Jiang2016}, approximate bisimulation~\cite{CastroPanangadenPrecup_2009,castro2021mico}, and approximate information states (AIS)~\cite{subramanian2022approximate} (which is used in our proof of \autoref{thm:approx}). Although there are various RL algorithms based on such state abstractions, the key difference is that all these frameworks focus on stationary policies in the infinite horizon setting. 
Our key insight that non-stationary/periodic policies improve performance is also applicable to these frameworks.

\textbf{ASQL for POMDPs.} As stated earlier, \ref{eq:ASQL} may be viewed as the special case of \ref{eq:PASQL} when $L=1$.
The convergence of the simplest version of ASQL was established in~\cite{singh1994learning} for $Z_t = Y_t$ under the assumption that the actions are chosen i.i.d.\ (and do not depend on $z_t$). In~\cite{Perkins2002} it was established that $Q^{0}_{\mu}$ is the fixed point of~\eqref{eq:ASQL}, but convergence of $\{Q_t\}_{t \ge 1}$ to $Q^{0}_{\mu}$ was not established. The convergence of ASQL when the agent state is a finite window memory was established in~\cite{Kara2022}. These results were generalized to general agent-state models in~\cite{EWRL-RQL}. The regret of an optimistic variant of ASQL was presented in~\cite{dong2022simple}. However, all of these papers focus on stationary policies. 

Our analysis is similar to the analysis of~\cite{Kara2022,EWRL-RQL} with two key differences. First, their convergence results were derived under the assumption that the learning rates are the reciprocal of visitation counts. We relax this assumption to the standard learning rate conditions of \autoref{ass:lr} using ideas from stochastic approximation. Second, their analysis is restricted to stationary policies. We generalize the analysis to periodic policies using ideas from time-periodic Markov chains.

\textbf{Q-learning for non-Markovian environments.} As highlighted earlier, a key challenge in understanding the convergence of \ref{eq:PASQL} is that the agent-state is not Markovian. The same conceptual difficulty arises in the analysis of Q-learning for non-Markovian environments~\cite{majeed2018q,chandak2024reinforcement,Kara2024}. Consequently, our analysis has stylistic similarities with the analysis in~\cite{majeed2018q,chandak2024reinforcement,Kara2024} but the technical assumptions and the modeling details are different. And more importantly, they restrict attention to stationary policies. Given our results, it may be worthwhile to explore if periodic policies can help in non-Markovian environments as well.

\textbf{Continual learning and non-stationary MDPs.} Non-stationarity is an important consideration in continual learning (see~\cite{abel2024definition} and references therein). However, in these settings, the environment is non-stationary. Our setting is different: the environment is stationary, but non-stationary policies help because the agent state is not Markov.

\textbf{Hierarchical learning.}
The options framework~\cite{precup2000temporal,sutton1999between,dietterich2000hierarchical,bacon2017option} is a hierarchical approach that learns temporal abstractions in MDPs and POMDPs. 
Due to temporal abstraction, the policy learned by the options framework is non-stationary. The same is true for other hierarchical learning approaches proposed in~\cite{wiering1997hq,chane2021goal,vezhnevets2017feudal}.
In principle, \ref{eq:PASQL} could be considered as a form of temporal abstraction where time is split into trajectories of length $L$ and then a policy of length~$L$ is learned. However, the theoretical analysis for options is mostly restricted to MDP setting and the convergence guarantees for options in POMDPs are weaker~\cite{steckelmacher2018reinforcement,qiao2018pomdp,le2018deep}. Nonetheless, the algorithmic tools developed for options might be useful for \ref{eq:PASQL} as well.

\textbf{Double Q-learning.} The update equation of \ref{eq:PASQL} are structurally similar to the update equations used in double Q-learning~ \cite{hasselt2010double, van2016deep}. However, the motivation and settings are different: the motivation for Double Q-learning is to reduce overestimation bias in off-policy learning in MDPs, while the motivation for \ref{eq:PASQL} is to induce non-stationarity while learning in POMDPs. Therefore, the analysis of the two algorithms is very different. More importantly, the end goals differ: double Q-learning learns a stationary policy while \ref{eq:PASQL} learns a periodic policy. 

\textbf{Use of non-stationary/periodic policies in MDPs} is investigated in~\cite{scherrer2012use,lesner2015,bertsekas2013abstract} in the context of approximate dynamic programming (ADP). Their main result was to show that using non-stationary or periodic policies can improve the approximation error in
ADP. Although these results use periodic policies, the setting of ADP in MDPs is very different from ours.

\section{Discussion} \label{sec:discussion}

\vskip -0.8\baselineskip

\textbf{Deterministic vs.\ stochastic policies.}
In this work, we restricted attention to periodic deterministic policies. In principle, we could have also considered periodic \emph{stochastic} policies. For stationary policies (i.e., when period is one), stochastic policies can outperform deterministic policies \cite{singh1994learning} as illustrated by \autoref{ex:2} in \appref{app:stochastic}. However, we do not consider stochastic policies in this work because we are interested in understanding Q-learning with agent-state and Q-learning results in a deterministic policy. There are two options to obtain stochastic policies: using regularization~\cite{geist2019theory},  
which changes the objective function; or using policy gradient algorithms~\cite{Sutton1999,Baxter2001}, which are a different class of algorithms than Q-learning. 

However, as illustrated in the motivating \autoref{ex:1} presented in the introduction, non-stationary policies can do better than stationary stochastic policies as well. So, adding non-stationarity via periodicity remains an interesting research direction when learning stochastic policies as well. 

\textbf{PASQL is a special case of ASQL with state augmentation.}
In principle, \ref{eq:PASQL} could be considered as a special case of \ref{eq:ASQL} with an augmented agent state $\bar Z_t = (Z_t, \MOD{t})$. However, the convergence analysis of~\ref{eq:ASQL} in~\cite{Kara2022,EWRL-RQL} does not imply the convergence of~\ref{eq:PASQL} because the results of~\cite{Kara2022,EWRL-RQL} are derived under the assumption that Markov chain $\{(S_t, Y_t, Z_t, A_t)\}_{t \ge 1}$ is irreducible and aperiodic, while we assume that the Markov chain is \emph{periodic}. Due to our weaker assumption, we are able to establish convergence of~\ref{eq:PASQL} to time-varying periodic policies.

\begin{wrapfigure}{R}{5cm}
    \vskip -1.5\baselineskip
    \centering
    \includegraphics[width=5cm]{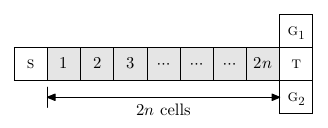}
    \caption{A T-shaped grid world. Agent starts at $\textsc{s}$, where it learns whether the goal state is $\textsc{g}_1$ or $\textsc{g}_2$. It has to go through the corridor $\{1,\dots,2n\}$, without knowing where it is, reach $\textsc{t}$ and go up or down to reach the goal state.}
    \label{fig:ex3}
\end{wrapfigure}
\textbf{Non-stationary policies vs.\ memory augmentation.}
Non-stationarity is a fundamentally different concept than memory augmentation.  As an illustration, consider the T-shaped grid world (first considered in \cite{NIPS2001_a38b1617}) shown in \autoref{fig:ex3}, which has a corridor of length~$2n$. In \appref{app:ex3}, we show that for this example, a stationary policy which uses a sliding window of past~$m$ observations and actions as the agent state needs a memory of at least $m > 2n$ to reach the goal state. In contrast, a periodic policy with period $L=3$ can reach the goal state for every~$n$. This example shows that periodicity is a different concept from memory augmentation and highlights the fact that mechanisms other than memory augmentation can achieve optimal behavior. 

The analysis of this paper is applicable to general memory augmented policies, so we do not need to choose between memory augmentation and periodicity. Our main message is that once the agent's memory is fixed based on practical considerations, adding periodicity could improve performance.

\textbf{Choice of the period $L$.} 
If the agent state $Z_t$ is a good approximation to the belief state, then \ref{eq:ASQL} (or, equivalently, \ref{eq:PASQL} with $L=1$) would converge to an approximately optimal policy. So, using \ref{eq:PASQL} a period $L > 1$ is useful when the agent state is not a good approximation of the belief state. 

As shown by \autoref{ex:1} in the introduction, the performance of the best periodic policy does not increase monotonically with the period~$L$. However, if we consider periods in the set $\{ n! : n \in \naturalnumbers \}$, then the performance increases monotonically. However, \ref{eq:PASQL} does not necessarily converge to the best periodic policy. The quality of the converged policy \eqref{eq:policy} depends on the behavior policy~$\mu$. The difficulty of finding a good behavioral policy increases with~$L$. In addition, increasing the period increases the memory required to store the tuple $(Q^0, \dots, Q^L)$ and the number of samples needed to converge (because each component is updated only once every $L$ samples). Therefore, the choice of the period~$L$ should be treated as a hyperparameter that needs to be tuned.

\textbf{Choice of the behavioral policy.} The behavioral policy impacts the converged limit of \ref{eq:PASQL}, and consequently it impacts the periodic greedy policy that is learned. As we pointed out in the discussion after \autoref{thm:convergence}, this dependence is a fundamental consequence of using an agent state that is not Markov and cannot be avoided. Therefore, it is important to understand how to choose behavioral policies that lead to convergence to good policies. 

\textbf{Generalization to other variants.} Our analysis is restricted to tabular off-policy Q-learning where a fixed behavioral policy is followed. Our proof fundamentally depends on the fact that the behavioral policy induces a cyclic limiting distribution on the periodic Markov chain $\{(S_t,Y_t,Z_t, A_t)\}_{t \ge 1}$. Such a condition is not satisfied in variants such as $\epsilon$-greedy Q-learning and SARSA. Generalizing the technical proof to cover these more practical algorithms (including function approximation) is an important future direction.

\section*{Acknowledgments}

The work of AS and AM was supported in part by a grant from Google’s Institutional Research Program in collaboration with Mila. The numerical experiments were enabled in part by support provided by Calcul
Qu\'ebec and Compute Canada.

\newpage

\printbibliography

\newpage

\appendix
\let\addcontentsline\originaladdcontentsline
\setcounter{tocdepth}{2} 
\renewcommand{\contentsname}{Contents of Appendix}

\tableofcontents

\newpage

\section{Illustrative examples}

\subsection{\autoref{ex:PASQL-example}: Learning curves for \ref{eq:ASQL}}\label{app:ex:PASQL-example}

For each behavioral policy $\bar \mu_k$, $k \in \ALPHABET K$, we run \ref{eq:PASQL} for 25 random seeds. The median + interquantile range of the iterates $\{Q_t(z,a)\}_{t \ge 1}$ as well as the theoretical limits $Q_{\bar \mu_k}(z,a)$ (computed as per \autoref{thm:convergence} for $L=1$) are shown in \autoref{fig:plots-period1}. These curves show that the result of \autoref{thm:convergence} is valid for the stationary case ($L=1$) as well.

\begin{figure}[!h]
    \centering
    \includegraphics{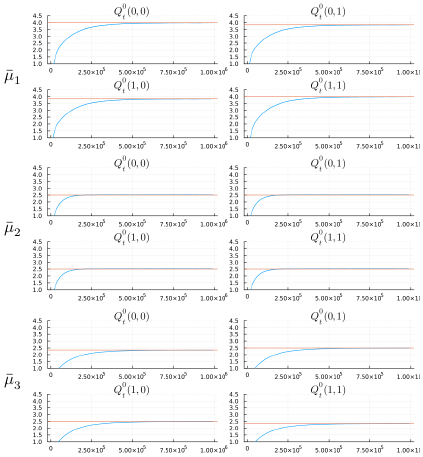}
    \caption{\ref{eq:ASQL} iterates for different behavioral policies (in blue) and the limit predicted by \autoref{thm:convergence} (in red).}
    \label{fig:plots-period1}
\end{figure}

\subsection{\autoref{ex:1}: non-stationary policies can outperform stationary policies}\label{app:ex1}

\begin{example}\label{ex:1}
   Consider a POMDP with $\ALPHABET S = \integers_{> 0}$, $\ALPHABET A = \{0,1\}$, and $\ALPHABET Y = \{0,1\}$. The system starts in an initial state $s_1 = 1$ and has deterministic dynamics.  To describe the dynamics and the reward function, we define $\ALPHABET D_0 \coloneqq \{ n(n+1)/2 + 1 : n \in \integers_{\ge 0} \}$, $\ALPHABET D_1 = \naturalnumbers \setminus \ALPHABET D_0$, and $\ALPHABET D = \ALPHABET D_0 \times \{0\} \cup \ALPHABET D_1 \times \{1\} \subset \ALPHABET S \times \ALPHABET A$. Then, the dynamics, observations, and rewards are given by
   \begin{equation}
       s_{t+1} = \begin{cases}
           s_t + 1, & (s_t, a_t) \in \ALPHABET D, \\
           1,       & \hbox{otherwise},
       \end{cases}
       \quad
       y_t = \begin{cases}
           0, & \hbox{$s_t$ is odd}, \\
           1, & \hbox{$s_t$ is even}, 
       \end{cases}
       \quad
       r(s,a) = \begin{cases}
           +1, & (s,a) \in \ALPHABET D, \\
           -1  & \hbox{otherwise}.
       \end{cases}
   \end{equation}
   Thus, the state is incremented if the agent takes action~$0$ when the state is in $\ALPHABET D_0$ and takes action~$1$ when the state is in $\ALPHABET D_1$. Taking these actions yield a reward of~$+1$. Not taking such an action results in a reward of $-1$ and resets the state to~$1$. The agent does not observe the state, but only observes whether the state is odd or even. A graphical representation of the model is shown in \autoref{fig:ex1}.
\end{example}

\begin{figure}[!h]
    \renewcommand\thefigure{\ref*{fig:ex1}}
    \centering
    \includegraphics{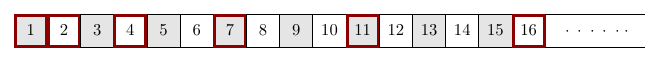}
    \caption{Graphical representation of \autoref{ex:1}. The cells indicate the state of the environment. Cells with the same background color have the same observation. The cells with a thick red boundary correspond to elements of the set $\ALPHABET D_0 \coloneqq \{ n(n+1)/2 + 1 : n \in \naturalnumbers \}$, where the action~$0$ gives a reward of $+1$ and moves the state to the right, while the action~$1$ gives a reward of $-1$ and resets the state to~$1$. The cells with a thin black boundary correspond to elements of the set $\ALPHABET D_1 = \naturalnumbers \setminus \ALPHABET D_0$, where the action~$1$ gives the reward of $+1$ and moves the state to the right while the action~$0$ gives a reward of $-1$ and resets the state to~$1$. Discount factor $\gamma = 0.9$.}
    \label{fig:ex1-app}
\end{figure}
\addtocounter{figure}{-1}

\textbf{For policy class $\PiHND$ (the class of all belief-based deterministic policies),} since the system starts in a known initial state and the dynamics are deterministic, the agent can compute the current state (thus, the belief is a delta function on the current state). Thus, the agent can always choose the correct action depending on whether the state is in $\ALPHABET D_0$ and $\ALPHABET D_1$. Hence $\JHND = 1/(1-\gamma)$, which is the highest possible reward.

\textbf{For policy class $\PiZSD$ (the class of all agent-state based deterministic policies),} there are four possible deterministic policies. For odd observations, the agent may take action $0$ and $1$. Similarly, for even observations, the agent may take action $0$ or $1$. Note that the system starts in state $1$, which is in $\ALPHABET D_0$. Therefore, if the agent chooses action~$1$ when the observation is odd, it receives a reward of $-1$ and stays at state~$1$. Therefore, the discounted total reward is $-1/(1-\gamma)$, which is the least possible value. Therefore, any policy that chooses $1$ on odd observations cannot be optimal. Therefore, the optimal (deterministic) action on odd observations is to pick action~$0$. Thus, there are two policies that we need to evaluate. 
\begin{itemize}
    \item If the agent chooses action $0$ at both odd and even observations, the state cycles between $1 \to 2 \to 3 \to 1 \to 2 \to 3 \cdots$ with the reward sequence $(+1,+1,-1,+1,+1,-1,\dots)$. Thus, the cumulative total reward of this policy is $(1+\gamma-\gamma^2)/(1-\gamma^3)$.
    \item If the agent chooses action $0$ at odd observations and action $1$ at even observations, the state cycles between $1 \to 2 \to 1 \to 2 \cdots$ with the reward sequence $(+1,-1,+1,-1,\dots)$. Thus, the cumulative total reward of this policy is $1/(1+\gamma)$. 
\end{itemize}
It is easy to verify that for $\gamma \in (0,1)$, $1/(1+\gamma) < (1+\gamma-\gamma^2)/(1-\gamma^3)$. Thus, 
\begin{equation}
    \JZSD = \frac{1 + \gamma - \gamma^2}{1 - \gamma^3}.
\end{equation}

We also consider \textbf{policy class $\PiZSS$: the class of all stationary stochastic agent-state based policies.} For policy class $\PiZSS$, the policy is characterized by two numbers $(p_0, p_1) \in [0,1]^2$, where $p_y$ denotes the probability of choosing action~$1$ when the observation is $y$, $y \in \{0,1\}$. We compute the approximately optimal policy by doing a brute force search over $(p_0,p_1)$ by discretizing them two decimal places and for each choice, running a Monte Carlo simulation of length $1,000$ and averaging it over $100$ random seeds. We find that there is negligible difference between the performance of stochastic and deterministic policies.

Finally, we consider \textbf{policy class $\Pi_{L}$}, which is the class of periodic deterministic agent-state based policies. A policy $\pi \in \Pi_L$ is characterized by two vectors $p_0, p_1 \in \{0,1\}^L$, where $p_{y,\ell}$ denotes the action chosen when $t \bmod L = \ell$ and the observation is $y$. We do an exhaustive search over all deterministic policies of length $L$, $L \in \{1,\dots,10\}$ to compute the numbers shown in the main text.

\subsection{\autoref{ex:2}: stochastic policies can outperform deterministic policies}\label{app:stochastic}

When the agent state is not an information state, the optimal stochastic stationary policy will perform better than (or equal to) the optimal deterministic stationary policy as observed in~\cite{singh1994learning}. Here is an example to illustrate this for a simple toy POMDP.

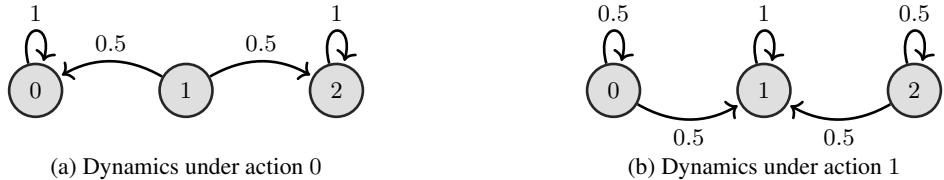
\begin{figure}[!ht]
    \centering
    \begin{subfigure}{0.45\linewidth}
    \centering
    {\begin{tikzpicture}[line width=1pt, node distance=2cm and 2cm,font=\small]
        \useasboundingbox (-0.5,-0.65) rectangle (4.5,1.25);
        \node[state] (M) {$0$};
        \node[state, right of=M] (Z) {$1$};
        \node[state, right of=Z] (P) {$2$};

        \draw[every loop]
            (Z) edge[bend right, auto=right] node {$0.5$} (M)
            (Z) edge[bend left, auto=left] node {$0.5$} (P)
            (M) edge[loop above] node {$1$} (M)
            (P) edge[loop above] node {$1$} (P)
            ;

    \end{tikzpicture}}
    \caption{Dynamics under action~$0$}
    \end{subfigure}
    \hfill
    \begin{subfigure}{0.45\linewidth}
    \centering
    {\begin{tikzpicture}[line width=1pt, node distance=2cm and 2cm,font=\small]
        \useasboundingbox (-0.5,-0.65) rectangle (4.5,1.25);
        \node[state] (M) {$0$};
        \node[state, right of=M] (Z) {$1$};
        \node[state, right of=Z] (P) {$2$};

        \draw[every loop]
            (M) edge[bend right, auto=right] node {$0.5$} (Z)
            (P) edge[bend left, auto=left] node {$0.5$} (Z)
            (M) edge[loop above] node {$0.5$} (M)
            (P) edge[loop above] node {$0.5$} (P)
            (Z) edge[loop above] node {$1$} (Z)
            ;
            
    \end{tikzpicture}}
    \caption{Dynamics under action~$1$}
    \end{subfigure}
    \caption{The dynamics for \autoref{ex:2}.}
    \label{fig:ex2}
\end{figure}

\begin{example}\label{ex:2}
Consider a POMDP with $\ALPHABET S = \{0, 1, 2\}$, $\ALPHABET A = \{0, 1\}$ and $\ALPHABET Y = \{0\}$. The system starts at an initial state $s_1 = 0$ and the dynamics under the two actions are shown in \autoref{fig:ex2}. The agent does not observe the state, i.e., $Y_t \equiv 0$. The rewards under action~$0$ are
$r(\cdot,0) = [-1, 0, 2]$ and the rewards under action~$1$ are $r(s,1) = -0.5$, for all $s \in \ALPHABET S$. 
\end{example}

\newpage

\begin{wrapfigure}{r}{0.5\textwidth}
    \centering
    \vskip-1.25\baselineskip
    \includegraphics[width=0.5\textwidth]{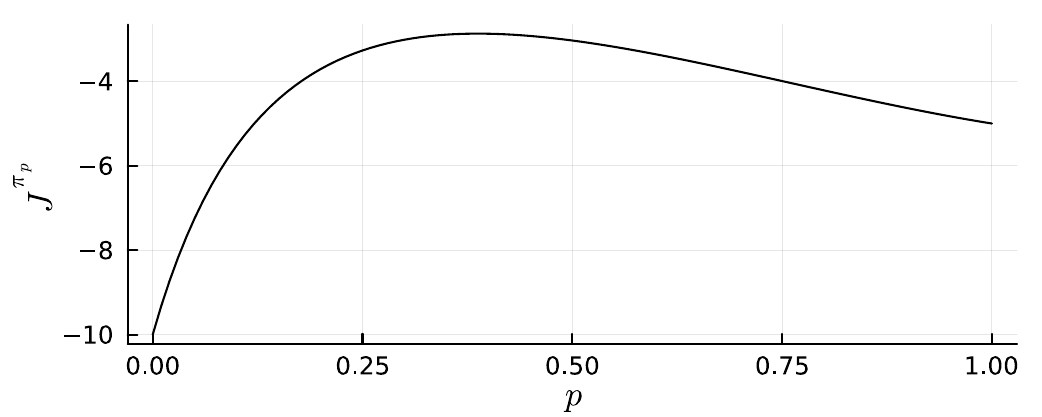}
    \vskip-0.25\baselineskip
    \caption{Performance of stationary stochastic policies ${\pi}_p$ for $p \in [0,1]$ for \autoref{ex:2}.}
    \label{fig:ex2-perf}
    \vskip-1\baselineskip
\end{wrapfigure}
We consider agent state $Z_t = Y_t$. Let $\PiZSS$ denote the  of all stationary stochastic policies and $\PiZSD$ denote the class of of all stationary deterministic policic A policy $\pi \in \PiZSS$ is parameterized by a single parameter $p \in [0,1]$, which indicates the probability of choosing action~$1$. We denote such a policy by $\pi_p$. Note that $p \in \{0,1\}$, $\pi_p \in \PiZSD$. Let $(P_a, r_a)$ denote the probability transition matrix and reward function when $a \in \ALPHABET A$ is chosen and let $(P_p, r_p) = (1-p)(P_0, r_0) + p(P_1, r_1)$. Then, the performance of policy $\pi_p$ is given by $J^{\pi_p} = [(1 - \gamma P_p)^{-1} r_p]_{0}$. The performance for all $p \in [0,1]$ for $\gamma = 0.9$ is shown in \autoref{fig:ex2-perf}, which shows that the best performance is achieved by the stochastic policy $\pi_p$ with $p \approx 0.39$.

Thus, stochastic policies can outperform deterministic policies.

\subsection{\autoref{ex:3}: conceptual difference between state-augmentation and periodic policies}
\label{app:ex3}

\begin{figure}[!ht]
    \renewcommand\thefigure{\ref*{fig:ex3}}
    \centering
    \includegraphics{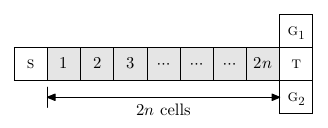}
    \caption{A T-shaped grid world for \autoref{ex:3}. In state~$\textsc{s}$, the agent learns about the goal state. In states $\{1, 2, \dots, 2n\}$, the agent simply knows that it is in the gray corridor, but does not know which cell it is in. In state $\textsc{t}$, it knows that it has reached the end of corridor and must decide whether to go up or down. The agent gets a reward of $+1$ for reaching the correct goal state and a reward of $-1$ for reaching the wrong goal state.}
    \label{fig:ex3-app}
\end{figure}
\addtocounter{figure}{-1}

\begin{example}\label{ex:3}
Consider a T-shaped grid world showed in \autoref{fig:ex3-app} with state space $\ALPHABET P \times \ALPHABET G$, where  $\ALPHABET P = \{ \textsc{s}, 1, 2, \dots, 2n, \textsc{t} \}$ is the position of the agent and $\ALPHABET G = \{\textsc{g}_1, \textsc{g}_2\}$ is the location of the goal. The observation space is $\ALPHABET Y = \{ 0, 1, 2, 3\}$. The observation is a deterministic function of the state and is given as follows:
\begin{itemize}
    \item At state $(\textsc{s}, \textsc{g}_i)$, $i \in \{1,2\}$, the observation is $i$ and reveals the location of the goal state to the agent. 
    \item At states $\{1, \dots, 2n\} \times \ALPHABET G$, the observation is $0$, so the agent cannot distinguish between these states.
    \item At states $\{ \textsc{t} \} \times \ALPHABET G$, the observation is $3$, so the agent knows when it reaches the $\textsc{t}$ state.
\end{itemize}

The action space depends on the current state: actions $\{\textsc{left}, \textsc{right}, \textsc{stay} \}$ are available when the agent is at $\{\textsc{s}, 1, \dots, 2n\}$ and actions $\{\textsc{up}, \textsc{down}\}$ are available at position $\textsc{t}$. 

The agent receives a reward of $+1$ if it reaching the goal state and $-1$ if it reaches the wrong goal state state (i.e., reaches $\textsc{g}_2$ when the goal state is $\textsc{g}_1$). The discount factor $\gamma = 1$.
\end{example}

We consider two classes of policies: 
\vspace*{-0.75\baselineskip}
\begin{enumerate}[label=(\roman*)]
    \item $\PiZSD(m)$: Stationary policies with agent state equal to a sliding window of the last $m$ observations and actions. 
    \item $\Pi_L$: Periodic policies with agent state equal to the last observation and periodic~$L$. 
\end{enumerate}

It is easy to see that as long as the window length $m \le 2n$, any policy in $\PiZSD(m)$ yields an average return of~$0$; for window lengths $m > 2n$, the agent can remember the first observation, and therefore it is possible to construct a policy that yields a return of $+1$.

We now consider a deterministic periodic policy with period $L=3$ given as follows:%
\footnote{For the ease of notation, we start the system at time $t=0$.}
$\pi =  (\pi^0, \pi^1, \pi^2)$ where $\pi^{\ell} \colon \ALPHABET Y \to \ALPHABET A$. We denote each $\pi^{\ell}$ as a column vector, where the $y$-th component indicates the action $\pi^\ell(y)$, where -- means that the choice of the action for that observation is irrelevant for performance.  The policy is given by
\begin{equation}
   \pi^0 = \MATRIX{
   \textsc{right} \\
   \textsc{right} \\
   \textsc{stay} \\
   \textsc{stay}
   },
   \quad
   \pi^1 = \MATRIX{
   \textsc{right} \\
   \text{--} \\
   \textsc{right} \\
   \textsc{up}
   },
   \quad
   \pi^2 = \MATRIX{
   \textsc{stay} \\
   \text{--} \\
   \text{--} \\
   \textsc{down}
   }.
\end{equation}
It is easy to verify if the system starts in state $(0,\textsc{g}_1)$, then by following policy $(\pi^0, \pi^1, \pi^2)$, the agent reaches state $\textsc{g}_1$ at time $3n + 3$. Moreover, when the system starts in state $(0,\textsc{g}_2)$, then by following the policy $(\pi^0, \pi^1, \pi^2)$, the agent reaches $\textsc{g}_2$ at time $3n + 4$. Thus, in both cases, the policy $(\pi^0, \pi^1, \pi^2)$ yields the maximum reward of $+1$.

\section{Periodic Markov chains}\label{app:periodic-MC}

In most of the standard reference material on Markov chains, it is assumed that the Markov chain is aperiodic and irreducible. In our analysis, we need to work with periodic Markov chains. In this appendix, we review some of the basic properties of Markov chains and then derive some fundamental results for periodic Markov chains.

Let $\ALPHABET S$ be a finite set. A stochastic process $\{S_t\}_{t \ge 0}$, $S_t \in \ALPHABET S$, is called a \textbf{Markov chain} if it satisfies the \emph{Markov property}: for any $t \in \integers_{\ge 0}$ and $s_{1:t+1} \in \ALPHABET S^{t+1}$, we have
\begin{equation}\label{eq:Markov}
    \PR(S_{t+1} = s_{t+1} \mid S_{1:t} = s_{1:t}) = \PR(S_{t+1} = s_{t+1} \mid S_t = s_t).
\end{equation}

If is often convenient to assume that $\ALPHABET S = \{1,\dots, n\}$. We can define an $n \times n$ transition probability matrix $P_t$ given by
\(
    [P_t]_{ij} = \PR(S_{t+1} = j \mid S_t = i)
\). Then, all the probabilistic properties of the Markov chain is described by the transition matrices $(P_0, P_1, \dots)$. 

In particular, suppose the Markov chain starts at the initial PMF (probability mass function) $\xi_0$ and let $\xi_t$ denote the PMF at time~$t$. We will view $\xi_t$ as a $n$-dimensional row vector. Then, Eq.~\eqref{eq:Markov} implies $\xi_{t+1} = \xi_t P_t$ and, therefore, 
\begin{equation}
    \xi_{t+1} = \xi_0 P_0 P_1 \cdots P_t.
\end{equation}

\subsection{Time-homogeneous Markov chains and their properties}
A Markov chain is said to be \textbf{time-homogeneous} if the transition matrix $P_t$ is the same for all time~$t$.  In this section, we state some standard results for time-homogeneous Markov chains~\cite{norris1998}.

\subsubsection{Classification of states}
The states of a time-homogeneous Markov chain can be classified as follows.
\begin{enumerate}
    \item 
    We say that a state $j$ is \textbf{accessible from} $i$ (abbreviated as $i \rightsquigarrow j$) if there is exists an $m \in \integers_{\ge 0}$ (which may depend on $i$ and $j$) such that $[P^m]_{ij} > 0$. The fact that $[P^m]_{ij} > 0$ implies that there exists an ordered sequence of states $(i_0, \dots, i_m)$ such that $i_0 = i$ and $i_m = j$ such that $P_{i_k i_{k+1}} > 0$; thus, there is a path of positive probability from state~$i$ to state~$j$.
    
    Accessibility is an transitive relationship, i.e., if $i \rightsquigarrow j$ and $j \rightsquigarrow k$ implies that $i \rightsquigarrow k$.

    \item 
    Two distinct states $i$ and $j$ are said to \textbf{communicate} (abbreviated to $i \leftrightsquigarrow j$) if $i$ is accessible from $j$ (i.e., $j \rightsquigarrow i$) and $j$ is accessible from $i$ ($i \rightsquigarrow j$). Alternatively, we say that $i$ and $j$ communicate if there exist $m, m' \in \integers_{\ge 0}$ such that $[P^{m}]_{ij} > 0$ and $[P^{m'}]_{ji} > 0$.
    
    Communication is an equivalence relationship, i.e., it is reflexive ($i \leftrightsquigarrow i$), symmetric ($i \leftrightsquigarrow j$ if and only if $j \leftrightsquigarrow i$), and transitive ($i \leftrightsquigarrow j$ and $j \leftrightsquigarrow k$ implies $i \leftrightsquigarrow k$).

    \item 
    The states in a finite-state Markov chain can be partitioned into two sets: \textbf{recurrent states} and \textbf{transient states}. A state is recurrent if it is accessible from all states that are from it (i.e., $i$ is recurrent if $i \rightsquigarrow j$ implies that $j \rightsquigarrow i$). States that are not recurrent are \textbf{transient}.

    It can be shown that a state $i$ is recurrent if and only if
    \begin{equation}
        \sum_{t=1}^{\infty} [ P^t ]_{ii} = \infty. 
    \end{equation}

    \item 
    States $i$ and $j$ are said to belong to the same \textbf{communicating class} if $i$ and $j$ communicate. Communicating classes form a partition the state space. Within a communicating class, all states are of the same type, i.e., either all states are recurrent (in which case the class is called a recurrent class) or all states are transient (in which case the class is called a transient class).

    A Markov chain with a single communicating class (thus, all states communicate with each other and are, therefore, recurrent) is called \textbf{irreducible}.

    \item 
    The \textbf{period} of a state $i$, denoted by $d(i)$, is defined as
    \begin{equation}
        d(i) = \gcd\{ t \in \integers_{\ge 1} : [P^t]_{ii} > 0 \}.
    \end{equation}
    If the period is $1$, the state is \textbf{aperiodic}, and if the period is $2$ or more, the state is \textbf{periodic}. It can be shown that all states in the same class have the same period. 

    A Markov chain is aperiodic, if all states are aperiodic. A simple sufficient (but not necessary) condition for an irreducible Markov chain to be aperiodic is that there exists a state $i$ such that $P_{ii} > 0$. In general, for a finite and aperiodic Markov chain, there exists a positive integer $T$ such that 
    \begin{equation}
        [P^t]_{ii} > 0, 
        \quad \forall t \ge T, i \in \ALPHABET S.
    \end{equation}
\end{enumerate}

\subsubsection{Limit behavior of Markov chains}

We now state some special distributions for a time-homogeneous Markov chain.
\begin{enumerate}
    \item 
    A PMF $\zeta$ on $\ALPHABET S$ is called a \textbf{stationary distribution} if $\zeta = \zeta P$. Thus, if a (time-homogeneous) Markov chain starts in a stationary distribution, it stays in a stationary distribution. 

    A finite irreducible Markov chain has a unique stationary distribution. Moreover, when the Markov chain is also aperiodic, the stationary distribution is given by $\zeta(j) = 1/m_j$, where $m_j$ is the expected return time to state~$j$.
    
    \item 
    A PMF $\zeta$ on $\ALPHABET S$ is called a \textbf{limiting distribution} if 
    \begin{equation}
        \lim_{t \to \infty} [ P^t]_{ij} = \zeta(j),
        \quad \forall i,j \in \ALPHABET S.
    \end{equation}

    A finite irreducible Markov chain has a limiting distribution if and only if it is aperiodic. Therefore, for an aperiodic Markov chain, the limiting distribution is the same as the stationary distribution.
\end{enumerate}

\begin{theorem}[Strong law of large numbers for Markov chains, Theorem 5.6.1 of~\cite{Durrett2019}]\label{thm:SLLN}
    Suppose $\{S_t\}_{t \ge 1}$ is an irreducible Markov chain that starts in state $i \in \ALPHABET S$. Then,
    \begin{equation}\label{eq:SLLN}
        \lim_{T \to \infty} \frac 1T \sum_{t=0}^{T-1} \IND \{ S_t = j \} = \frac 1{m_j}.
    \end{equation}
    Therefore, for any function $h \colon \ALPHABET S \to \reals$, 
    \begin{equation}\label{eq:ergodic}
        \lim_{T \to \infty} \frac 1T \sum_{t=0}^{T-1} h(S_t) = \sum_{j \in \ALPHABET S} \frac {h(j)}{m_j}.
    \end{equation}
    If, in addition, the Markov chain $\{S_t\}_{t \ge 1}$ is aperiodic, and has a limiting distribution $\zeta$, then we have that
    \begin{equation}\label{eq:ergodic-aperiodic}
        \lim_{T \to \infty} \frac 1T \sum_{t=0}^{T-1} h(S_t) = \sum_{j \in \ALPHABET S} \zeta(j) h(j).
    \end{equation}
\end{theorem}

\subsection{Time-varying with periodic transition matrix}

In this section, we consider time-varying Markov chains where the transition matrices $(P_0, P_1, \dots)$ are periodic with period $L$. Let $\MOD{t} = (t \bmod L)$ and $\ALPHABET L = \{0, \dots, L-1\}$. Then, the transition matrix $P_t$ is the same as $P_{\MOD{t}}$. Thus, the system dynamics are completely described by the transition matrices $\{P_{\ell}\}_{\ell \in \ALPHABET L}$. With a slight abuse of notation, we will call such a Markov chain as \textbf{$L$-periodic Markov chain}. We will show later that the notion of \emph{time-periodicity} that we are considering is equivalent to the notion of \emph{state-periodicity} for time-homogeneous Markov chains defined earlier. 

\subsection{Constructing an equivalent time-homogeneous Markov chain}
Since the Markov chain is not time-homogeneous, the classification and results of the previous section are not directly applicable. There are two ways to construct a time-homogeneous Markov chain: using state augmentation or viewing the process after every $L$ steps. 

\subsubsection{Method 1: State augmentation}
The original time-varying Markov chain $\{S_t\}_{t\ge 0}$ is equivalent to the time-homogeneous Markov chain $\{ (S_t, \MOD{t}) \}_{t \ge 0}$ defined on $\ALPHABET S \times \ALPHABET L$ with transition matrix $\bar P$ given by
\begin{equation}\label{eq:time-homogeneous}
    \bar P( (s', \ell') \mid (s, \ell) ) = P_{\ell}(s' \mid s) \IND\{ \ell' = \MOD{\ell + 1} \}.
\end{equation}

\begin{example}\label{ex:MC}
Consider a $2$-periodic Markov chain with state space $\ALPHABET S = \{1, 2\}$ and transition matrices
\begin{equation}
    \def\1{\tfrac 14}
    \def\2{\tfrac 34}
    P_0 = \MATRIX{ \1 & \2 \\[5pt] \frac 12 & \frac 12 }
    \quad\text{and}\quad
    P_1 = \MATRIX{ \2 & \1 \\[5pt] \1 & \2 }.
\end{equation}
\end{example}
The time-periodic Markov chain of \autoref{ex:MC} may be viewed as a time-homogeneous Markov chain with state space $\{1, 2\} \times \{0, 1\}$ and transition matrix
\begin{equation}
    \def\1{\tfrac 14}
    \def\2{\tfrac 34}
    \bar P =
    \bbordermatrix{
           & (1,0) & (2,0) & (1,1) & (2,1) \cr 
    (1,0)  &   0   &   0   &  \1   &  \2   \cr 
    \noalign{\vskip 5pt}
    (2,0)  &   0   &   0   &  \frac 12   &  \frac 12   \cr 
    \noalign{\vskip 5pt}
    (1,1)  &   \2  &   \1  &  0    &   0   \cr 
    \noalign{\vskip 5pt}
    (2,1)  &   \1  &   \2  &  0    &   0   \cr
    }
    =
    \MATRIX{0 & I \\ I & 0 } \MATRIX{ P_0 & 0 \\ 0 & P_1 }
\end{equation}
where $0$ denotes the all zero matrix and $I$ denotes the identity matrix (both of size $2 \times 2$). Note that the time-homogeneous Markov chain is periodic.
    
Define the following:
\begin{itemize}
    \item $L$ block diagonal matrices $\Lambda_0, \dots, \Lambda_{L-1} \in \reals^{nL \times nL}$ as follows:
    \begin{equation}
        \Lambda_0 = \diag(P_0, P_1,\dots, P_{L-1}),
        \quad
        \Lambda_1 = \diag(P_{L-1}, P_0, \dots, P_{L-2}),
        \quad
        \text{ etc.}
    \end{equation}
    
    \item A permutation matrix $\Pi \in \{0, 1\}^{nL \times nL}$ as follows
    \begin{equation}
        \Pi = 
        \MATRIX{ 0 & I & \cdots & 0 \\
                 \vdots & \ddots & \ddots & \vdots \\
                 0 & 0 & \cdots & I \\
                 I & 0 & \cdots & 0
               }
    \end{equation}
    where each block is $n \times n$.
\end{itemize}
The permutation matrix $\Pi$  satisfies the following properties (which can be verified by direct algebra):
\begin{enumerate}
    \item[(P1)] $\Pi \, \Pi^\TRANS = I$ and therefore $\Pi^{-1} = \Pi^\TRANS$.
    \item[(P2)] $\Pi^L = I$.
    \item[(P3)] $\Lambda_\ell \Pi = \Pi \Lambda_{\MOD{\ell+1}}$, $\ell \in \ALPHABET L$.
\end{enumerate} 

In general, the transition matrix of the Markov chain $\{(S_t, \MOD{t})\}_{t \ge 0}$ is 
\begin{equation}
    \bar P =  
    \MATRIX{ 0 & P_0 & \cdots & 0 \\
             \vdots & \ddots & \ddots & \vdots \\
             0 & 0 & \cdots & P_{L-2} \\
             P_{L-1} & 0 & \cdots & 0
           }_{nL \times nL}
    = \Lambda_0 \Pi.
\end{equation}

\subsubsection{Method 2: Viewing the process every $L$ steps}

The original Markov chain viewed every $L$-steps, i.e., the process $\{ S_{kL + \ell} \}_{k \ge 0}$, $\ell \in \ALPHABET L$, is a time-homogeneous Markov chain with transition probability matrix $\mathcal P_\ell$ given by
\begin{equation}\label{eq:Pell}
    \mathcal P_{\ell} = P_{\MOD{\ell}} P_{\MOD{\ell + 1}} \cdots P_{\MOD{\ell + L - 1}}
\end{equation}
that is,
\begin{equation}
    \mathcal P_0 = P_0 P_1 \cdots P_{L-2} P_{L-1},
    \quad
    \mathcal P_1 = P_1 P_2 \cdots P_{L-1} P_{0},
    \quad
    \text{etc.}
\end{equation}

\subsubsection{Relationship between the two constructions}
The two constructions are related as follows.
\begin{proposition}\label{prop:periodic}
    We have that
    $
        \bar P^L = \diag(\mathcal P_0, \dots, \mathcal P_L).
    $
\end{proposition}
\begin{proof}
    From (P3), we get that $\bar P = \Pi \Lambda_1$. Therefore, 
    \begin{equation}
        \bar P^2 = 
        \Lambda_0 \Pi \Lambda_0 \Pi
        =
        \Lambda_0 \Lambda_1 \Pi^2
    \end{equation}
    Similarly
    \begin{equation}
        \bar P^3 = \Lambda_0 \Pi \bar P^2  = \Lambda_0 \Pi \Lambda_0 \Lambda_1 \Pi^2
        = \Lambda_0 \Lambda_1 \Pi \Lambda_1 \Pi^2
        = \Lambda_0 \Lambda_1 \Lambda_2 \Pi^3
    \end{equation}
    Continuing this way, we get
    \begin{equation}
        \bar P^L = \Lambda_0 \Lambda_1 \dots \Lambda_{L-1} \Pi^L
        = \Lambda_0 \Lambda_1 \dots \Lambda_{L-1} .
    \end{equation}
    where the last equality follows from (P2).
    The result then follows from the definitions of $\Lambda_{\ell}$ and $\mathcal P_{\ell}$, $\ell \in \ALPHABET L$.
    \qed
\end{proof}

\subsection{Limiting behavior of periodic Markov chain}

In the subsequent discussion, we consider the following assumptions.
\begin{assumption}\label{ass:MC-assumptions}
    Every $\{ \mathcal P_\ell \}$, $\ell \in \ALPHABET L$, is irreducible and aperiodic
\end{assumption}

Suppose \autoref{ass:MC-assumptions} holds. Define $\zeta^{\ell}$ to be the unique stationary distribution for Markov chain $\mathcal P_{\ell}$, $\ell \in \ALPHABET L$, i.e., $\zeta^{\ell}$ is the unique PMF that satisfies $\zeta^{\ell} = \zeta^{\ell} \mathcal P_{\ell}$. 

\begin{proposition}\label{prop:stationary}
    The PMFs $\{\zeta^\ell\}_{\ell \in \ALPHABET L}$ satisfy
    \begin{equation}
        \zeta^{\ell} P_{\ell} = \zeta^{\MOD{\ell + 1}},
        \quad \ell \in \ALPHABET L.
    \end{equation}
\end{proposition}
\begin{proof}
    We prove the result for $\ell = 0$. The analysis is the same for general $\ell$. By assumption, we have that
    \begin{equation}
        \zeta^0 = \zeta^0 \mathcal P_0 = \zeta^0 P_0 P_1 \cdots P_{L-1}.
    \end{equation}
    Let $\bar \zeta^1 \coloneqq \zeta^0 P_0$. Then, we have
    \begin{equation}
        \bar \zeta^1 = \zeta^0 P_0 = \zeta^0 P_0 P_1 \cdots P_{L-1} P_0 = \bar \zeta^1 P_1 \cdots P_{L-1} P_0 = \bar \zeta^1 \mathcal P_1.
    \end{equation}
    Thus $\bar \zeta^1$ is a stationary distribution.  Since $\mathcal P_1$ is irreducible, the stationary distribution is unique, hence $\bar \zeta^1$ must equal $\zeta^1$.
    \qed
\end{proof}

We can verify this result for \autoref{ex:MC}. For this model, we have
\begin{equation}
    \mathcal P_0 = P_0 P_1 = \MATRIX{ \frac 38 & \frac 58 \\[5pt] \frac 12 & \frac 12 }
    \quad\text{and}\quad
    \mathcal P_1 = P_1 P_0 = \MATRIX{ \frac 5{16} & \frac {11}{16} \\[5pt] \frac 7{16} &\frac 9{16} }.
\end{equation}
Thus, 
\begin{equation}
    \zeta^0 = \MATRIX{ \frac 49 & \frac 59}
    \quad\text{and}\quad
    \zeta^1 = \MATRIX{ \frac {7}{18} & \frac {11}{18}}
\end{equation}
And we can verify that $\zeta^0 P_0 = \zeta^1$ and $\zeta^1 P_1 = \zeta^0$.

\begin{proposition}\label{prop:limit-cycle}
    Under \autoref{ass:MC-assumptions}, the limiting distribution of the Markov chain $\{S_t\}_{t \ge 0}$ is cyclic. In particular, for any initial distribution $\xi_0$, 
    \begin{equation}\label{eq:limiting}
        \lim_{k \to \infty} \xi_{kL + \ell} = \zeta^\ell
    \end{equation}
    Furthermore, 
    \begin{equation}\label{eq:SLLN-periodic}
        \limsup_{K \to \infty} \frac 1K \sum_{k=0}^{K-1} \IND\{ S_{kL + \ell} = i \} = [\zeta^\ell]_i,
        \quad \forall i \in \ALPHABET S, \ell \in \ALPHABET S.
    \end{equation}
    Consequently, for any function $h \colon \ALPHABET S \to \reals$, 
    \begin{equation}\label{eq:SLLN-ergodic}
        \limsup_{K \to \infty} \frac 1K \sum_{k=0}^{K-1} h(S_{kL + \ell}) = \sum_{s \in \ALPHABET S}h(s)[\zeta^\ell]_s,
        \quad \ell \in \ALPHABET S.
    \end{equation}
\end{proposition}
\begin{proof}
The results follow from standard results for the time-homogeneous Markov chain $\{S_{kL+\ell}\}_{k \ge 0}$.
    \qed
\end{proof}

\begin{proof}[Alternative]
    We present an alternative proof that uses the state augmented Markov chain $\bar P$. We first prove that under \autoref{ass:MC-assumptions}, the chain $\bar P$ is irreducible periodic with period~$L$. 

    The proof of irreducibility relies on two observations.
    \begin{enumerate}
        \item Fix an $\ell \in \ALPHABET L$ and consider $i,j \in \ALPHABET S$.  Since $\mathcal P_{\ell}$ is irreducible, we have that there exists a positive integer $m$ (depending on $i$, $j$, and $\ell$) such that $[\mathcal P_{\ell}^{m}]_{ij} > 0$. Note that \autoref{prop:periodic} implies that $[\bar P^{mL}]_{(i,\ell),(j,\ell)} = [\mathcal P_{\ell}]_{ij} > 0$. Therefore, in the Markov chain $\bar P$, states $(i,\ell) \rightsquigarrow (j,\ell)$. Since $i$ and $j$ were arbitrary, all states $\ALPHABET S \times \{ \ell \}$ belong to the same communicating class. 

        \item 
        Now consider two $\ell, \ell' \in \ALPHABET L$. Suppose we start at some state $(i,\ell) \in \ALPHABET S \times \{\ell\}$, then in $[\ell' - \ell]$ steps, we will reach some state $(j, \ell') \in \ALPHABET S \times \{\ell'\}$. Thus, $(j,\ell')$ is accessible from $(i,\ell)$. But, we have already argued that all states in $\ALPHABET S \times \{\ell\}$ belong to the same communicating class, therefore all states in $\ALPHABET S \times \{\ell'\}$ are accessible from all states in $\ALPHABET S \times \{\ell\}$. By interchanging the roles of $\ell$ and $\ell'$, we have that all states in $\ALPHABET S \times \{\ell\}$ are accessible from all starts in $\ALPHABET S \times \{\ell'\}$. Therefore, the states $\ALPHABET S \times \{\ell\}$ and $\ALPHABET S \times \{\ell'\}$ belong to the same communicating class. Since $\ell$ and $\ell'$ were arbitrary, we have that all states of $\bar P$ belong to the same communicating class. Hence, $\bar P$ is irreducible.
    \end{enumerate}

    We now show that $\bar P$ is periodic. First observe that the Markov chain starting in the set $\ALPHABET S \times \{ \ell \}$ does not return to the same set for the first $L-1$ steps. Thus, $[\bar P^{t}]_{(i,\ell),(i,\ell)} = 0$ for $t \in \{1, 2, \dots, L-1\}$. Therefore, the only possible values of $t$ for which $[\bar P^t]_{(i,\ell),(i,\ell)} > 0$ are those that are multiples of $L$. Hence, for any $(i,\ell) \in \ALPHABET S \times \ALPHABET L$, 
    \begin{equation}\label{eq:gcd-barP}
        d(i,\ell) = \gcd\{ t \in \integers_{\ge 1} : [\bar P^t]_{(i,\ell),(i,\ell)} > 0 \}
        = L  \gcd\{ k \in \integers_{\ge 1} : [\mathcal P_{\ell}^k]_{ii} > 0 \} 
    \end{equation}
    Moreover, since $\mathcal P_{\ell}$ is aperiodic, $\gcd\{ k \in \integers_{\ge 1} : [\mathcal P_{\ell}^k]_{ii} > 0 \} = 1$. Substituting in~\eqref{eq:gcd-barP}, we get that $d(i,\ell) = L$ for all $(i,\ell)$. Thus, all states have a period of $L$.

    Now, from \autoref{prop:periodic}, we know that $\bar P^L = \diag(\mathcal P_0, \dots, \mathcal P_{L-1})$. Therefore
    \begin{equation}\label{eq:limit-Pbar}
        \lim_{k \to \infty} [\bar P^{kL}]_{(i,\ell),(j,\ell)} = [\zeta^{\ell}]_j, 
        \quad (i,\ell) \in \ALPHABET S \times \ALPHABET L.
    \end{equation}
    Consequently, if we start with an initial distribution $\bar \xi_0$ such that $\bar \xi_0(\ALPHABET S \times \{0\}) = 1$, then,
    \begin{equation}
        \lim_{k \to \infty} \bar \xi_{kL} = \VEC(\zeta_0, 0, \dots, 0)
    \end{equation}
    where the $0$ vectors are of size~$n$.
    Consequently, \autoref{prop:stationary} implies that
    \begin{equation}
        \lim_{k \to \infty} \bar \xi_{kL + \ell} = \VEC(0, \dots, 0, \zeta_{\ell}, 0, \dots, 0),
        \quad \forall \ell \in \ALPHABET L
    \end{equation}
    where  $\zeta^{\ell}$ is the $\ell$-th place. This completes the proof of~\eqref{eq:limiting}.

    Now consider the function $\bar h \colon \ALPHABET S \times \ALPHABET L \to \reals$ defined as 
    $\bar h(s,\ell') = h(s) \IND\{\ell' = \ell\}$. Then, by taking $T = KL$, we have
    \begin{equation}
        \lim_{K \to \infty} \frac 1{K} \sum_{t=0}^{K-1} h(S_{kL + \ell})
        =
        \lim_{T \to \infty} \frac LT \sum_{t=0}^{T-1} \bar h(S_t, \MOD{t}) 
        =
        L \sum_{s \in \ALPHABET S} \frac {h(s)}{m_{(s,\ell)}}
    \end{equation}
    where the last equation uses~\eqref{eq:ergodic} from \autoref{thm:SLLN}. Now, \eqref{eq:SLLN-ergodic} follows from observing that mean return time to state $(s,\ell)$ in Markov chain $\bar P$ is $L$ times the mean-return time to state~$s$ in Markov chain $\mathcal P_{\ell}$, which equals $1/[\zeta^\ell]_{s}$ since $\mathcal P_\ell$ is irreducible and aperiodic.
    \qed
\end{proof}

\section{Periodic Markov decision processes}\label{app:periodic-MDP}

Periodic MDPs are a special class time non-stationary MDPs where the dynamics and rewards are periodic. In particular, let $\mathcal M$ be a time-varying MDP with state space $\ALPHABET S$, action space $\ALPHABET A$, and dynamics and reward at time~$t$ given by $P_t \colon \ALPHABET S \times \ALPHABET A \to \Delta(\ALPHABET S)$ and $r_t \colon \ALPHABET S \times \ALPHABET A \to \reals$.

As before, we use $\MOD{t}$ to denote $t \bmod L$ and $\ALPHABET L$ to denote $\{0, \dots, L-1\}$. The MDP $\mathcal M$ is periodic with period $L$ if there exist $(P^\ell, r^\ell)$, $\ell \in \ALPHABET L$ such that for all~$t$:
\begin{equation}
    P_t(S_{t+1} \mid S_t, A_t) = P^{\MOD{t}}(S_{t+1} \mid S_t, A_t)
    \quad\text{and}\quad
    r_t(S_t,A_t) = r^{\MOD{t}}(S_t,A_t).
\end{equation}
Periodic MDPs were first considered in~\cite{Riis1965}. Periodic MDPs may be viewed as stationary MDPs by considering the augmented state $(S_t, \MOD{t})$. By this equivalence, it can be shown that there is no loss of optimality in restricting attention to periodic policies. In particular, let $(V^0, \dots, V^{L-1})$ denote the fixed point of the following system of equations
\begin{equation}\label{eq:periodic-MDP-app}
    V^{\ell}(s) = \max_{a \in \ALPHABET A} \Bigl\{ r^{\ell}(s,a) 
    + \gamma \sum_{s' \in \ALPHABET S} P^{\ell}(s'|s,a) V^{\MOD{\ell+1}}(s') \Bigr\}, 
    \quad \forall (\ell,s,a) \in \ALPHABET L \times \ALPHABET S \times \ALPHABET A.
\end{equation}
Define $\pi^{\ell}_{\star}(s)$ to be the arg-max or the right hand side of~\eqref{eq:periodic-MDP-app}. Then the time-varying policy $\pi = (\pi_1, \pi_2, \dots)$ given by $\pi_t = \pi^{\MOD{t}}_{\star}$ is optimal. 

See~\cite{Scherrer2016} for a discussion of how to modify standard MDP algorithms to solve periodic dynamic program~\eqref{eq:periodic-MDP-app}.

\section{Stochastic Approximation with Markov noise}\label{app:SA}

We now state a generalization of \autoref{thm:SLLN} to stochastic approximation style iterations. 
\begin{theorem}\label{thm:SA-Markov}
    Let $\{S_t\}_{t \ge 1}$, $\ALPHABET S$, be an irreducible and aperiodic finite Markov chain with unique limiting distribution $\zeta$. Let $\mathcal F_t$ denote the natural filtration w.r.t.\ $\{S_t\}_{t \ge 1}$ and $\{\alpha_t\}_{t \ge 1}$ be a non-negative real-valued process adapted to $\{\mathcal F_t\}$ that satisfies
    \begin{equation}\label{eq:robbins-monro}
        \sum_{t \ge 1} \alpha_t = \infty
        \quad\text{and}\quad
        \sum_{t \ge 1} \alpha_t^2 < \infty.
    \end{equation}
    Let $\{M_{t+1}\}_{t \ge 1}$ be a square-integrable margingale difference sequence w.r.t.\ $\{\mathcal F_t\}_{t \ge 1}$ such that $\EXP[ M^2_{t+1} \mid \mathcal F_t] \le K(1 + \NORM{X_t}^2)$ for some constant $K$.
    Consider the iterative process $\{X_t\}_{t \ge 1}$, where $X_1$ is arbitrary and for $t \ge 1$, we have
    \begin{equation}\label{eq:SA-Markov}
        X_{t+1} = (1 - \alpha_t) X_t + \alpha_t \bigl[ h(S_t) + M_{t+1} \bigr]. 
    \end{equation}
    Then, the sequence $\{X_t\}_{t \ge 1}$ converges almost surely to limit. In particular, 
    \begin{equation}\label{eq:SA-Marokv-limit}
        \lim_{T \to \infty} X_T = \sum_{s \in \ALPHABET S} h(s) \zeta(s), 
        \quad a.s.
    \end{equation}
\end{theorem}

Eq.~\eqref{eq:SA-Markov} is similar to standard stochastic approximation iteration~\cite{robbins-monro:51,Kushner1997,Borkar2009}, which the ``noise sequence'' $h(S_t)$ is assumed to be a martingale difference sequence. The setting considered above is sometimes referred to as stochastic approximation with Markov noise. In fact, more general version of this result where the noise sequence is allowed to depend on the state $X_t$ are typically established in the literature~\cite{Benveniste2012,Borkar2009,Kushner1997,Bhatnagar2024}. For the sake of completeness, we will show that \autoref{thm:SA-Markov} is a special case of these more-general results. 

Before presenting the proof, we point out that
\autoref{thm:SA-Markov} is a generalization of \autoref{thm:SLLN}, Eq.~\eqref{eq:ergodic-aperiodic}. In particular, suppose the learning rates are
\(
    \alpha_t = 1/(1+t)
\). Then, simple algebra shows that
\begin{equation}
    X_T = \frac{1}{T} \sum_{t=1}^T h(S_t).
\end{equation}
Then, \eqref{eq:ergodic-aperiodic} of \autoref{thm:SLLN} implies that the limit is given by the right had side of~\eqref{eq:SA-Marokv-limit}. Therefore, \autoref{thm:SA-Markov} is a generalization of \autoref{thm:SLLN} to general learning rates which satisfy~\eqref{eq:robbins-monro}.

\begin{proof}
    To establish the result, we will show that the iteration $\{X_t\}_{t \ge 1}$ satisfies the assumptions for the convergence of stochastic approximation with (state dependent) Markov noise and stochastic recursive inclusions given in \cite[Theorem 2.7]{Bhatnagar2024}. The proof is due to \cite{Shalabh-email}. In particular, we can rewrite~\eqref{eq:SA-Markov} as
    \begin{equation}\label{eq:SA-alternative}
        X_{t+1} = X_t + \alpha_t g(X_t, S_t)
    \end{equation}
    where $g(x,s) = -x + h(s)$. Moreover, for ease of notation, define $\bar h = \sum_{s \in \ALPHABET S}h(s)\zeta(s)$. Then, we have
    \begin{itemize}
        \item $g(x,s)$ is Lipschtiz continuous in the first argument, so A2.14 of~\cite{Bhatnagar2024} holds.
        \item From \eqref{eq:ergodic-aperiodic}, the ergodic occupation measure of $\{h(S_t)\}_{t \ge 1}$ is $\{\bar h\}$, which is compact and convex. So, A2.15 of~\cite{Bhatnagar2024} is satisfied.
        \item The conditions on the martingale noise sequence $\{M_t\}_{t \ge 1}$ imply that A2.16 of~\cite{Bhatnagar2024} holds.
        \item Eq.~\eqref{eq:robbins-monro} is equivalent to A2.17 of \cite{Bhatnagar2024}.
        \item To check A2.18 of~\cite{Bhatnagar2024}, for any measure $\nu$ on $\ALPHABET S$, define
        \begin{equation}
            \tilde h(x,\nu) = \int g(x,s) \nu(ds) = -x + \bar h.
        \end{equation}
        Also define
        \begin{equation}
            \tilde h_c(x,\nu) =\frac{\tilde h(cx, c\nu)}{c} = -x + \frac{\bar h}{c}
        \end{equation}
        Let $\tilde h_{\infty}(x,\nu) = \lim_{c \to \infty} h_c(x,\nu) = x$.  Thus, the differential inclusion in A2.18(ii) is actually an ODE
        \begin{equation}
           \dot x = -x  
        \end{equation}
        which has origin as the unique global asymptotically stable equilibrium point. Thus, A2.18 of~\cite{Bhatnagar2024} is satisfied.
    \end{itemize}
    Therefore, all assumptions of Theorem 2.7 of \cite{Bhatnagar2024} are satisfied. Therefore, by that result, the iterates $\{X_t\}_{t \ge 1}$ converge to solution of the ODE (note that the differential inclusion in Theorem~2.7 of~\cite{Bhatnagar2024} is an ODE in our setting)
    \begin{equation}\label{eq:ODE}
        \dot x = - x + \bar h.
    \end{equation}
    Note that $x = \bar h$ is the unique asymptotically stable attractor of the ODE~\eqref{eq:ODE}. Therefore, Theorem~2.7 of \cite{Bhatnagar2024} implies~\eqref{eq:SA-Marokv-limit}.
    \qed
\end{proof}

\autoref{thm:SA-Markov} also implies the following generalization of \autoref{prop:limit-cycle}. 
\begin{proposition}\label{prop:SA-Markov-periodic}
    Suppose $\{S_t\}_{t \ge 1}$ is a time-periodic Markov chain with period $L$ that satisfies \autoref{ass:MC-assumptions} with the unique limiting distribution $\{\zeta^\ell\}_{\ell \in \ALPHABET L}$. Let $\{\mathcal F_t\}_{t \ge 1}$ denote the natural filtration w.r.t.\ $\{S_t\}_{t \ge 1}$ and $\{\alpha^\ell_t\}_{t \ge 1}$, $\ell \in \ALPHABET L$, be non-negative real-valued processes adapted to $\{\mathcal F_t\}_{t \ge 1}$ such that $\alpha^\ell_t = 0$ when $\ell \neq \MOD{t}$ and
    \begin{equation}\label{eq:robbins-monro-periodic}
        \sum_{t \ge 1} \alpha^{\ell}_t = \infty
        \quad\text{and}\quad
        \sum_{t \ge 1} (\alpha_t^{\ell})^2 < \infty.
    \end{equation}
    Let $\{M_{t+1}\}_{t \ge 1}$ be a square-integrable margingale difference sequence w.r.t.\ $\{\mathcal F_t\}_{t \ge 1}$ such that $\EXP[ M^2_{t+1} \mid \mathcal F_t] \le K(1 + \NORM{X_t}^2)$ for some constant $K$.
    Fix any $\ell \in \ALPHABET L$, Consider the iterative process $\{X^\ell_k\}_{k \ge 1}$, where $X_1$ is arbitrary and for $k \ge 1$, we have
    \begin{equation}\label{eq:SA-Markov-periodic}
        X^{\ell}_{t + 1} = (1 - \alpha^{\ell}_{t}) X^{\ell}_{t} + \alpha^{\ell}_{t} \bigl[ h(S_{t}) + M_{t+1} \bigr].
    \end{equation}
    Then, the sequence $\{X^\ell_t\}_{t \ge 1}$ converges almost surely to the following limit
    \begin{equation}\label{eq:SA-Marokv-limit-periodic}
        \lim_{t \to \infty} X^\ell_t = \sum_{s \in \ALPHABET S} h(s) \zeta^\ell(s), 
        \quad a.s.
    \end{equation}
\end{proposition}
\begin{proof}
    Note that the learning rates used here can be viewed as the learning rates of $L$ separated stochastic iterations on a common timescale $t$. Each separate stochastic iteration $\ell \in \ALPHABET L$ is actually only updated once every $L$ steps on the timescale $t$. Because of the condition $\alpha^\ell_t = 0$ when $\ell \neq \MOD{t}$, each update is followed by $L-1$ ``pseudo''-updates where the learning rate is $0$. Therefore, each $X^{\ell}$ is updated only once every $L$ steps on timescale $t$.
    
    The result then follows immediately from \autoref{thm:SA-Markov} by considering the process $\{S_t\}_{t \ge 1}$ every $L$ steps for each $\ell \in \ALPHABET L$.
    \qed
\end{proof}

\section{\autoref{thm:convergence}: Convergence of periodic Q-learning}\label{app:PASQL}

The high-level idea of the proof is similar to~\cite{Kara2022} for ASQL when the agent state is a finite window of past observations and action. The key observation of \cite{Kara2022} is the following:
Consider an iterative process
\(
    X_{t+1} = (1 - \alpha_t) X_t + \alpha_t U_t
\)
with the learning rates $\alpha_t = 1/(1 + t)$. Then, $X_{t+1} = (X_0 + \sum_{\tau=1}^t U_t)/(1 + t)$. Then, if the process $\{U_t\}_{t \ge 1}$ has an ergodic limit (e.g., when $\{U_t\}_{t \ge 1}$ is a function of a Markov chain, see \autoref{thm:SLLN}), the process $\{X_t\}_{t \ge 1}$ converges to the ergodic limit of $\{U_t\}_{t \ge 1}$. We follow a similar idea but with the following changes:
\begin{itemize}
    \item Instead of assuming ``averaging'' learning rates (i.e., reciprocal of the number of visits), we allow for general learning rates of \autoref{ass:lr}.
    \item We account for the fact that that the ``noise'' is periodic.
\end{itemize}
The rest of the analysis then follows along the standard argument of convergence of Q-learning~\cite{Jaakkola1994,Kara2022,Kara2024}.

Define the error function $\Delta^\ell_{t+1} \coloneqq Q^\ell_{t+1} - Q^\ell_{\EXPL}$, for all $\ell \in \ALPHABET L$. To prove \autoref{thm:convergence}, it suffices to prove that $\NORM{\Delta^\ell_t} \to 0$ for all $\ell \in \ALPHABET L$, where $\NORM{\cdot}$ is the supremum-norm.
The proof proceeds in three steps.

\subsection{Step 1: State splitting of the error function}
Define $V_t^\ell(z) \coloneqq \max_{a \in \ALPHABET A} Q_t^\ell(z,a)$ and $V_{\EXPL}^\ell(z) \coloneqq \max_{a \in \ALPHABET A} Q^\ell_{\EXPL}(z,a)$, for all $\ell \in \ALPHABET L$, $z \in \ALPHABET Z$.  
We can combine~\eqref{eq:PASQL}, \eqref{eq:PDP}, and~\eqref{eq:periodic-MDP} as follows
\begin{equation}
    \Delta^\ell_{t+1}(z,a) = (1 - \alpha^\ell_t(z,a)) \Delta^\ell_t(z,a) 
    +
    \alpha^\ell_t(z,a) \bigl[
        U^{\ell,0}_t(z,a) + U^{\ell,1}_t(z,a) + U^{\ell,2}_t(z,a) 
    \bigr]
    \label{eq:state}
\end{equation}
where 
\begin{align}
    U^{\ell,0}_t(z,a)&\coloneqq  \left[ r(S_t,A_t) - r^\ell_{\EXPL}(z,a) \right] \IND_{\{ Z_t = z, A_t = a\}}, \\
   U^{\ell,1}_t(z,a)&\coloneqq \Bigl[ \gamma V_{\EXPL}^{\MOD{\ell+1}}(Z_{t+1}) - \gamma \sum_{z' \in \ALPHABET Z} P^{\ell}_{\EXPL}(z' | z,a) V_{\EXPL}^{\MOD{\ell+1}}(z') \Bigr]\IND_{\{ Z_t = z, A_t = a\}}, \\
    U^{\ell,2}_t(z,a) &\coloneqq \gamma V_t^{\MOD{\ell+1}}(Z_{t+1}) - \gamma V^{\MOD{\ell+1}}_\EXPL(Z_{t+1}).
\end{align}
Note that we have added extra indicator functions in the $U^{\ell,i}_t(z,a)$ terms, $i \in \{0,1\}$. This does not change the value of $\alpha^{\ell}_t(z,a) U^{\ell,i}_t(z,a)$ because the learning rates have the property that $\alpha^{\ell}_t(z,a) = 0$ if $(\ell, z,a) \neq (\MOD{t}, z_t,a_t)$ (see \autoref{ass:lr}).

For each $\ell \in \ALPHABET L$, Eq.~\eqref{eq:state} may be viewed as a linear system with state $\Delta^\ell_{t+1}$ and three inputs $U^{\ell,0}_t$,$U^{\ell,1}_t$ and $U^{\ell,2}_t$. We exploit the linearity of the system and split the state into three components: $\Delta^\ell_{t+1} = X^{\ell,0}_{t+1} + X^{\ell,1}_{t+1} + X^{\ell,2}_{t+1}$, where the three components evolve as follows:
\begin{equation}
\label{eq:split}
    X^{\ell,i}_{t+1}(z,a) = (1 - \alpha^\ell_t(z,a)) X^{\ell,i}_t(z,a) 
    +
    \alpha^\ell_t(z,a)   U^{\ell,i}_t(z,a),
    \quad i \in \{0, 1,2\}
\end{equation}
Linearity implies that~\eqref{eq:state} is equivalent to~\eqref{eq:split}. We will now separately show that $\NORM{X^{\ell,0}_t} \to 0$, $\NORM{X^{\ell,1}_t} \to 0$ and $\NORM{X^{\ell,2}_t} \to 0$. 

\subsection{Step 2: Convergence of component \texorpdfstring{$X^{\ell,0}_t$}{first component}}
Fix $(\ell,z_\circ,a_\circ) \in \ALPHABET L \times \ALPHABET Z \times \ALPHABET A$ and define
\begin{align*}
    h_r( S_t, Z_t, A_t; \ell,z_\circ,a_\circ) &= \bigl[ r(S_t,A_t) - r^{\ell}_{\mu}(z_\circ,a_\circ) \bigr] \IND_{\{ Z_t = z_\circ, A_t = a_\circ\}}.
\end{align*}
Then the process $\{ X^{\ell,0}_t (z_\circ, a_\circ) \}_{t \geq 1}$ is given by the stochastic iteration
\begin{equation}
    X^{\ell,0}_{t+1} (z_\circ, a_\circ) = (1 - \alpha^\ell_t(z_\circ,a_\circ)) X^{\ell,0}_t(z_\circ, a_\circ) + \alpha^\ell_t(z_\circ,a_\circ) h_r( S_t, Z_t, A_t; \ell,z_\circ,a_\circ),
\end{equation}
which is of the form~\eqref{eq:SA-Markov-periodic}.  The process $\{(S_t,Z_t,A_t)\}_{t \ge 1}$ is a periodic Markov chain and the learning rates $\{\alpha^{\ell}_t(z_\circ,a_\circ)\}_{t \ge 1}$ satisfy the conditions of \autoref{prop:SA-Markov-periodic} due to \autoref{ass:lr}. Therefore, \autoref{prop:SA-Markov-periodic} implies 
that $\{ X^{\ell,0}_t (z_\circ, a_\circ) \}_{t \geq 1}$ converges a.s. to the following limit
\begin{align*}
    \lim_{t \to \infty} X^{\ell,0}_t (z_\circ, a_\circ) &= \sum_{s,z,a \in \ALPHABET S \times \ALPHABET Z \times \ALPHABET A} \zexpl^{\ell}(s,z,a) h_r(s, z,a; \ell,z_\circ,a_\circ) \\
    &= \sum_{s,z,a \in \ALPHABET S \times \ALPHABET Z \times \ALPHABET A} \zexpl^{\ell}(s,z,a) \IND_{\{ z = z_\circ, a = a_\circ\}} \bigl[ r(s,a) - r^{\ell}_{\mu}(z_\circ,a_\circ) \bigr] \\
    &= \biggl[\sum_{s \in \ALPHABET S} \zexpl^{\ell}(s,z_\circ,a_\circ) r(s,a_\circ) \biggr] - \zexpl^{\ell}(z_\circ,a_\circ) r^{\ell}_{\mu}(z_\circ,a_\circ)
    \notag \\
    &= \biggl[ \sum_{s \in \ALPHABET S} \zexpl^{\ell}(s,z_\circ,a_\circ) r(s,a_\circ) \biggr] 
    - \biggl[ \sum_{s \in \ALPHABET S}\zexpl^{\ell}(z_\circ,a_\circ) \zexpl^{\ell}(s | z_\circ) r(s,a_\circ) \biggr]
    \notag \\
    &= \biggl[ \sum_{s \in \ALPHABET S} \zexpl^{\ell}(s,z_\circ) \mu(a_\circ|z_\circ) r(s,a_\circ) \biggr] - 
       \biggl[ \sum_{s \in \ALPHABET S}\zexpl^{\ell}(z_\circ) \mu(a_\circ|z_\circ) \zexpl^{\ell}(s | z_\circ) r(s,a_\circ) \biggr]
    \notag \\
    &= 0
\end{align*}
Hence, for all $(\ell,z_\circ,a_\circ)$, the process $\{X^{\ell,0}_t(z_\circ,a_\circ)\}_{t \ge 1}$ converges to zero almost surely.

\subsection{Step 3: Convergence of component \texorpdfstring{$X^{\ell,1}_t$}{second component}}

Let $W_t$ denote the tuple $(S_t,Z_t,A_t,S_{t+1},Z_{t+1},A_{t+1})$. Note that $\{W_t\}_{t\ge1}$ is also a periodic Markov chain and converges to a cyclic limiting distribution $\bar \zeta^{\ell}_{\mu}$, where 
\[
   \bar \zeta^{\ell}_{\mu}(s,z,a,s',z',a') = \zexpl^{\ell}(s,z,a) \sum_{y' \in \ALPHABET Y} P(s',y'|s,a) \IND_{\{z' = \phi(z,y',a)\}} \mu(a'|z').
\]
We use $\bar \zexpl^{\ell}(s,z,a, \ALPHABET S, \mathcal Z, \mathcal A)$ to denote the marginalization over the ``future states'' and a similar notation for other marginalizations. Note that $\bar \zexpl^{\ell}(s,z,a, \ALPHABET S, \mathcal Z, \mathcal A) = \zexpl^{\ell}(s,z,a)$. 

Fix $(\ell,z_\circ,a_\circ) \in \ALPHABET L \times \ALPHABET Z \times \ALPHABET A$ and define
\begin{equation}
    h_P( W_t; \ell,z_\circ,a_\circ) = 
        \Bigl[ \gamma V^{\MOD{\ell+1}}_{\mu}(Z_{t+1}) - 
        \gamma \sum_{\bar z \in \ALPHABET Z} P^{\ell}_{\mu}(\bar z|z_\circ,a_\circ) V^{\MOD{\ell+1}}_{\mu}(\bar z) \Bigr] \IND_{\{Z_t = z_\circ, A_t = a_\circ \}}
\end{equation}
Then the process $\{ X^{\ell,1}_t (z, a) \}_{t \geq 1}$ is given by the stochastic iteration
\begin{equation}
    X^{\ell,1}_{t+1} (z_\circ, a_\circ) = (1 - \alpha^\ell_t(z_\circ,a_\circ)) X^{\ell,1}_t(z_\circ, a_\circ) + \alpha^\ell_t(z_\circ,a_\circ)  h_P( W_t ; \ell,z_\circ,a_\circ).
\end{equation}
which is of the form~\eqref{eq:SA-Markov-periodic}.  As argued earlier, the process $\{W_t\}_{t \ge 1}$ is a periodic Markov chain.
Due to \autoref{ass:lr}, the learning rate $\alpha^\ell_t(z_\circ, a_\circ)$ is measurable with respect to the sigma-algebra generated by $(Z_{1:t}, A_{1:t})$ and is therefore also measurable with respect to the sigma-algebra generated by $W_{1:t}$. Combining this with \autoref{prop:SA-Markov-periodic} implies that the learning rates $\{ \alpha^{\ell}_t(z_\circ, a_\circ) \}_{t \geq 1}$ satisfy the conditions of \autoref{prop:SA-Markov-periodic}. Therefore, \autoref{prop:SA-Markov-periodic} implies that  
$\{ X^{\ell,1}_t (z_\circ, a_\circ) \}_{t \geq 1}$ converges a.s. to the following limit
\begin{align*}
    \hskip 1em & \hskip -1em 
     \lim_{t \to \infty} X^{\ell,1}_t (z_\circ, a_\circ)  \\
    &= \sum_{\substack{ s,z,a \in \ALPHABET S \times \ALPHABET Z \times \ALPHABET A \\ s',z',a' \in  \ALPHABET S \times \ALPHABET Z \times \ALPHABET A}}
    \bar \zeta^{\ell}_{\mu}(s,z,a,s', z', a')
    h_P( s,z,a, s', z', a' ; \ell,z_\circ,a_\circ) \\
    &= \sum_{\substack{ s,z,a \in \ALPHABET S \times \ALPHABET Z \times \ALPHABET A \\ s',z',a' \in  \ALPHABET S \times \ALPHABET Z \times \ALPHABET A}}
        \bar \zeta^{\ell}_{\mu}(s,z,a,s', z', a')
        \Bigl[ \gamma V^{\MOD{\ell+1}}_{\mu}(z') - 
        \gamma \sum_{\bar z \in \ALPHABET Z} P^{\ell}_{\mu}(\bar z|z_\circ,a_\circ) V^{\MOD{\ell+1}}_{\mu}(\bar z) \Bigr] \IND_{\{z = z_\circ, a = a_\circ \}}
    \\
    &= \gamma \biggl[ \sum_{z' \in \ALPHABET Z } 
       \bar \zeta^{\ell}_{\mu}( \ALPHABET S, z_\circ,a_\circ, \ALPHABET S, z',  \ALPHABET A) V^{\MOD{\ell+1}}_{\mu}(z') \biggr] 
       -
       \biggl[ \gamma \bar \zeta^{\ell}_{\mu}( \ALPHABET S, z_\circ,a_\circ, \ALPHABET S, \mathcal Z, \mathcal A) 
       \sum_{\bar z \in \ALPHABET Z} P^{\ell}_{\mu}(\bar z|z_\circ,a_\circ) V^{\MOD{\ell+1}}_{\mu}(\bar z) \biggr] 
    \\
    &= 0 
\end{align*}
where the last step follows from the fact that $\bar \zeta^{\ell}_{\mu}( \ALPHABET S, z_\circ,a_\circ, \ALPHABET S, \mathcal Z, \mathcal A) = \zexpl^{\ell}(z_\circ,a_\circ)$ and $\bar \zeta^{\ell}_{\mu}( \ALPHABET S, z_\circ,a_\circ, \ALPHABET S, z',  \ALPHABET A) = \zexpl^{\ell}(z_\circ,a_\circ) P^{\ell}_{\mu}(z'|z_\circ,a_\circ)$.

\subsection{Step 4: Convergence of component \texorpdfstring{$X^{\ell,2}_t$}{third component}}

The remaining analysis is similar to corresponding step in the standard convergence proof of Q-learning and its variations~\cite{Jaakkola1994,Kara2022,Kara2024}. In this section, we use $\NORM{\cdot}$ to denote the supremum norm, i.e., $\NORM{\cdot}_{\infty}$.

In the previous step, we have shown that $\NORM{X^{\ell,i}_t} \to 0$ a.s., for $i \in \{0, 1\}$. Thus, we have that $\NORM{X^{\ell,0}_t + X^{\ell,1}_t} \to 0$ a.s. Arbitrarily fix an $\epsilon > 0$. Therefore, there exists a set $\Omega^1$ of measure one and a constant $T(\omega, \epsilon)$ such that for $\omega \in \Omega^1$, all $t > T(\omega, \epsilon)$, and  $(\ell, z,a) \in \ALPHABET L \times \ALPHABET Z \times \ALPHABET A$, we have 
\begin{equation}\label{eq:X1-eps}
    X^{\ell,0}_t(z,a) + X^{\ell,1}_t(z,a) < \epsilon.
\end{equation}

Now pick a constant $C$ such that 
\begin{equation}\label{eq:C}
    \kappa \coloneqq \gamma \left( 1 + \frac 1C \right) < 1
\end{equation}
Suppose for some $t > T(\omega, \epsilon)$, $\max_{\ell \in \ALPHABET L} \NORM{X^{\ell,2}_t} > C \epsilon$. Then, for $(z,a) \in \ALPHABET Z \times \ALPHABET A$, 
\begin{subequations}
\begin{align}
    U^{\ell,2}_t(z,a) &= \gamma V_t^{\MOD{\ell+1}}(Z_{t+1}) - \gamma V^{\MOD{\ell+1}}_\EXPL(Z_{t+1})
    \\
    & = \gamma \max_{a \in \ALPHABET A} Q_t^{\MOD{\ell+1}}(Z_{t+1}, a) - \max_{a' \in \ALPHABET A} \gamma Q^{\MOD{\ell+1}}_\EXPL(Z_{t+1}, a')
    \\
    & \leq \gamma \max_{a \in \ALPHABET A} \left\{ Q_t^{\MOD{\ell+1}}(Z_{t+1}, a) - \gamma Q^{\MOD{\ell+1}}_\EXPL(Z_{t+1}, a) \right\}
    \\
    & \stackrel{(a)}\leq \gamma \NORM{Q^{\MOD{\ell+1}}_t - Q^{\MOD{\ell+1}}_\EXPL} = \gamma \NORM{\Delta^{\MOD{\ell+1}}_t}
     \\
    &\le \gamma \NORM{X^{\MOD{\ell+1},0}_t + X^{\MOD{\ell+1},1}_t} + \gamma \NORM{X^{\MOD{\ell+1},2}_t}
    \stackrel{(b)} \le \gamma \epsilon + \gamma \NORM{X^{\MOD{\ell+1},2}_t}
    \label{eq:U-bd2}
    \\
    &\stackrel{(c)}\le \gamma \left( 1 + \frac 1C \right) \max_{\ell \in \ALPHABET L} \NORM{X^{\ell,2}_t}
    \stackrel{(d)} = \kappa\max_{\ell \in \ALPHABET L} \NORM{X^{\ell,2}_t}
    \stackrel{(d)} < \max_{\ell \in \ALPHABET L} \NORM{X^{\ell,2}_t}.
    \label{eq:U-bd3}
    \\
\end{align}
\end{subequations}
where $(a)$ follows from the fact that an upper bound is obtained by maximizing over all realizations of $Z_{t+1}$, $(b)$ follows from \eqref{eq:X1-eps}, $(c)$ follows from the fact that $\max_{\ell \in \ALPHABET L} \NORM{X^{\ell,2}_t} > C \epsilon$, $(d)$ follows from \eqref{eq:C}. Thus, for any $t > T(\omega, \epsilon)$ and $\max_{\ell \in \ALPHABET L} \NORM{X^{\ell,2}_t} > C \epsilon$, we have

\begin{align*}
    X^{\ell,2}_{t+1}(z,a) 
    &=
    (1 - \alpha^{\ell}_t(z,a)) X^{\ell,2}_t(z,a)
    + 
    \alpha^\ell_t(z,a) U^{\ell,2}_t(z,a)
    < 
    \max_{\ell \in \ALPHABET L} \NORM{X^{\ell,2}_t} \\
    \implies \max_{\ell \in \ALPHABET L} \NORM{X^{\ell,2}_{t+1}} & < \max_{\ell \in \ALPHABET L} \NORM{X^{\ell,2}_t}.
\end{align*}

Hence, when $\max_{\ell \in \ALPHABET L} \NORM{X^{\ell,2}_t} > C\epsilon$, it decreases monotonically with time. Hence, there are two possibilities: 
either 
\begin{enumerate*}[label=(\roman*)]
    \item $\max_{\ell \in \ALPHABET L} \NORM{X^{\ell,2}_t}$ always remains above $C\epsilon$; or
    \item it goes below $C\epsilon$ at some stage.
\end{enumerate*}
We consider these two possibilities separately.

\subsubsection{Possibility (i): \texorpdfstring{$\max_{\ell \in \ALPHABET L} \NORM{X^{\ell,2}_t}$ always remains above $C\epsilon$}{}}

\newcommand\X{\NORM{X^{\ell,2}_t}}
\newcommand\U{\NORM{U^{\ell,2}_t}}
\newcommand\M[1]{\NORM{M^{\ell,(#1)}_t}}

We will show that $\max_{\ell \in \ALPHABET L} \NORM{X^{\ell,2}_t}$ cannot remain above $C\epsilon$ forever. We first start with a basic result for random iterations. This is a self-contained result, so we reuse some of the variables used in the rest of the paper.
\begin{lemma}\label{lem:sequence}
    Let $\{X_t\}_{t \ge 1}$, $\{Y_t\}_{t \ge 1}$, and $\{\alpha_t\}_{t \ge 1}$ be non-negative sequences adapted to a filtration $\{\mathcal F_t\}_{t \ge 1}$ that satisfy the following:
    \begin{subequations}
    \begin{align}
        X_{t+1} &\le (1 - \alpha_t) X_t , \label{eq:X-conv}\\
        Y_{t+1} &\le (1 - \alpha_t) Y_t + \alpha_t c, \label{eq:Y-conv}
    \end{align}
    \end{subequations}
    where $c$ is a constant. Suppose
    \begin{equation}\label{eq:alpha}
        \sum_{t = 1}^{\infty} \alpha_t = \infty
    \end{equation}
    Then, the sequence $\{X_t\}_{t \ge 1}$ converges to zero almost surely and the sequence $\{Y_t\}_{t \ge 1}$ converges to $c$ almost surely.
\end{lemma}
\begin{proof}
   The iteration~\eqref{eq:X-conv} implies that 
   \begin{equation}
       X_{t+1} \le \Bigl[ (1-\alpha_1) \cdots (1 - \alpha_t) \Bigr] X_1
   \end{equation}
   Condition~\eqref{eq:alpha} implies that the term in the square brackets converges to zero. Therefore, $X_t \to 0$. 

   Observe that the iteration~\eqref{eq:Y-conv} can be rewritten as
   \begin{equation}
       Y_{t+1} - c \le (1 - \alpha_t) (Y_t - c)
   \end{equation}
   which is of the form~\eqref{eq:X-conv}. Therefore, $Y_t -c \to 0$.
    \qed
\end{proof}

We will now prove that $\max_{\ell \in \ALPHABET L} \NORM{X^{\ell,2}_t}$ cannot remain above $C\epsilon$ forever.  The proof is by contradiction.
Suppose $\max_{\ell \in \ALPHABET L} \NORM{X^{\ell,2}_t}$ remains above $C\epsilon$ forever. As argued earlier, this implies that $\max_{\ell \in \ALPHABET L} \NORM{X^{\ell,2}_t}$, $t \ge T(\omega,\epsilon)$, is a strictly decreasing sequence, so it must be bounded from above. Let $B^{(0)}$ be such that $\max_{\ell \in \ALPHABET L} \NORM{X^{\ell,2}_t} \le B^{(0)}$ for all $t \ge T(\omega,\epsilon)$. Eq.~\eqref{eq:U-bd3} implies that $\U < \kappa B^{(0)}$. Then, we have that 
\begin{align*}
    \max_{\ell \in \ALPHABET L} X^{\ell,2}_{t+1} (z, a) &\le (1 - \alpha^\ell_t(z,a)) \max_{\ell \in \ALPHABET L} \NORM{X^{\ell,2}_t} + \alpha^\ell_t(z,a) \max_{\ell \in \ALPHABET L} \U \\
    &\le (1 - \alpha^\ell_t(z,a)) \max_{\ell \in \ALPHABET L} \NORM{X^{\ell,2}_t} + \alpha^\ell_t(z,a) \kappa \max_{\ell \in \ALPHABET L} \NORM{X^{\ell,2}_t}
\end{align*}
which implies that $\max_{\ell \in \ALPHABET L} \NORM{X^{\ell,2}_t} \le \M0$, where $\{M^{\ell,(0)}_t\}_{t \ge T(\omega,\epsilon)}$ is a sequence given by
\begin{equation}
    M^{\ell,(0)}_{t+1}(z,a)  \le (1 - \alpha^\ell_t(z,a)) M^{\ell,(0)}_t(z,a) + \alpha^\ell_t(z,a) \kappa B^{(0)}, \quad \forall (z,a) \in \ALPHABET Z \times \ALPHABET A.
    \label{eq:M0}
\end{equation}
\autoref{lem:sequence} implies that $M^{\ell,(0)}_t(z,a) \to \kappa B^{(0)}$ and hence $\M0 \to \kappa B^{(0)}$. Now pick an arbitrary $\bar \epsilon \in (0, (1-\kappa) C\epsilon)$. Thus, there exists a time $T^{(1)} = T^{(1)}(\omega, \epsilon, \bar \epsilon)$ such that for all $t > T^{(1)}$, $\M0 \le B^{(1)} \coloneqq \kappa B^{(0)} + \bar \epsilon$. Since $\max_{\ell \in \ALPHABET L} \NORM{X^{\ell,2}_t}$ is bounded by $\M0$, this implies that for all $t > T^{(1)}$, $\max_{\ell \in \ALPHABET L} \NORM{X^{\ell,2}_t} \le B^{(1)}$ and, by~\eqref{eq:U-bd3}, $\U \le \kappa B^{(1)}$. By repeating the above argument, there exists a time $T^{(2)}$ such that for all $t \ge T^{(2)}$, 
\begin{equation}
\max_{\ell \in \ALPHABET L} \NORM{X^{\ell,2}_t} \le B^{(2)} \coloneqq \kappa B^{(1)} + \bar \epsilon = \kappa^2 B^{(0)} + \kappa \bar \epsilon + \bar \epsilon, 
\end{equation}
and so on. By~\eqref{eq:C}, $\kappa < 1$ and $\bar\epsilon$ is chosen to be less than $C\epsilon$. So eventually, $B^{(m)} \coloneqq \kappa^m B^{(0)} + \kappa^{m-1} \bar \epsilon + \cdots + \bar \epsilon$ must get below $C\epsilon$ for some $m$, contradicting the assumption that $\max_{\ell \in \ALPHABET L} \NORM{X^{\ell,2}_t}$ remains above $C\epsilon$ forever.

\subsubsection{Possibility (ii): \texorpdfstring{$\max_{\ell \in \ALPHABET L} \NORM{X^{\ell,2}_t}$ goes below $C\epsilon$ at some stage}{}}

Suppose that there is some $t > T(\omega, \epsilon)$ such that $\max_{\ell \in \ALPHABET L} \NORM{X^{\ell,2}_t} < C \epsilon$. Then~\eqref{eq:U-bd2} implies that
\begin{equation}
    \NORM{U^{\ell,2}_t} \le \gamma \NORM{X^{\MOD{\ell+1},0}_t + X^{\MOD{\ell+1},1}_t} + \gamma \NORM{X^{\MOD{\ell+1},2}_t}
    \le \gamma \epsilon + \gamma C \epsilon < C\epsilon
\end{equation}
where the last inequality uses~\eqref{eq:C}. Therefore,
\begin{equation}
    \max_{\ell \in \ALPHABET L} X^{\ell,2}_{t+1}(z,a) \le 
    (1 - \alpha^\ell_t(z,a)) \max_{\ell \in \ALPHABET L} \NORM{X^{\ell,2}_t} + \alpha^\ell_t(z,a) \max_{\ell \in \ALPHABET L} \NORM{U^{\ell,2}_t}
    < C \epsilon
\end{equation}
where the last inequality uses the fact that both $\NORM{U^{\ell,2}_t}$ and $\max_{\ell \in \ALPHABET L} \NORM{X^{\ell,2}_{t+1}}$ are both below $C\epsilon$. Thus, we have that 
\begin{equation}
    \max_{\ell \in \ALPHABET L} X^{\ell,2}_{t+1}(z,a) < C\epsilon.
\end{equation}
Hence, once $\max_{\ell \in \ALPHABET L} \NORM{ X^{\ell,2}_{t+1} }$ goes below $C\epsilon$, it stays there.

\subsubsection{Implication}
We have show that for sufficiently large $t > T(\omega, \epsilon)$, $\max_{\ell \in \ALPHABET L} X^{\ell,2}_{t}(z,a) < C \epsilon$. Since $\epsilon$ is arbitrary, this means that for all realizations $\omega \in \Omega^1$, $\max_{\ell \in \ALPHABET L} \NORM{X^{\ell,2}_{t}} \to 0$. Thus, 
\begin{equation}\label{eq:X2-limit}
    \lim_{t \to \infty} \max_{\ell \in \ALPHABET L} \NORM{X^{\ell,2}_{t}} = 0, \quad a.s.
\end{equation}

\subsection{Putting everything together}
Recall that we defined $\Delta^\ell_t = Q^\ell_t - Q_{\EXPL}$ and in Step $1$, we split $\Delta^\ell_t = X^{\ell,0}_t + X^{\ell,1}_t + X^{\ell,2}_t$. Steps~$2$ and $3$ together show that $\NORM{X^{\ell,0}_t + X^{\ell,1}_t} \to 0$, a.s.\ and Step $3$ \eqref{eq:X2-limit} shows us that $\max_{\ell \in \ALPHABET L} \NORM{X^{\ell,2}_{t}} \to 0$, a.s. Thus, by the triangle inequality, 
\begin{equation}
\lim_{t \to \infty} \NORM{\Delta^\ell_t} \le 
\lim_{t \to \infty} \NORM{X^{\ell,0}_t + X^{\ell,1}_t} 
+
\lim_{t \to \infty} \NORM{X^{\ell,2}_t} 
= 0,
\end{equation}
which establishes that $Q^\ell_t \to Q_{\EXPL}$, a.s.

\section{\autoref{thm:approx}: Sub-optimality gap} \label{app:approx}

The high-level idea of proving \autoref{thm:approx} is as follows. \autoref{thm:convergence} shows that \ref{eq:PASQL} converges to a cyclic limit, which is the solution to a periodic MDP. Thus, the question of characterizing the sub-optimality gap is equivalent to the following. Given a PODMP $\mathcal P$, let $\mathcal M$ be a periodic agent-state based model that approximates the reward and the dynamics of $\mathcal P$ (in the sense of an approximate information state, as defined in~\cite{subramanian2022approximate}). Let $\hat \pi^\star$ be the optimal policy of model~$\mathcal M$. What is the sub-optimality gap when $\hat \pi^\star$ is used in the original POMDP~$\mathcal P$?

To answer such questions, a general framework of approximate information states was developed in~\cite{subramanian2022approximate} for both finite and infinite horizon models. However, we cannot directly used the results of~\cite{subramanian2022approximate} because the infinite horizon results there were restricted to stationary policies, while we are interested in the sub-optimality gap of periodic policies. 

Nonetheless, \autoref{thm:approx} can be proved by building on the existing results of~\cite{subramanian2022approximate}. In particular, we start by looking at finite horizon model rather than infinite horizon model. Then, as per \cite[Definition 7]{subramanian2022approximate}, the agent state process may be viewed as an approximate information state with approximation errors $\{(\varepsilon_t,\delta_t)\}_{t \ge 1}$, where 
\begin{align}
    \varepsilon_t &= \sup_{h_t, a_t}
    \Bigl\lvert \EXP [R_{t} \mid h_{t}, a_{t}] - \sum_{s \in \ALPHABET S}  r(s,a) \zexpl^{\MOD{t}}(s \mid z,a) \Bigr\rvert, \\
    \delta_t &= \sup_{h_{t}, a_{t}} d_{\F} (\PR(Z_{t+1} = \cdot \mid h_{t}, a_{t}), P^{\MOD{t}}_{\EXPL}(Z_{t+1} = \cdot |\sigma_{t}(h_{t}),a_{t})).
\end{align}
Let $V^{\vec \pi}_{t,T}(h_t) = \EXP^{\vec \pi}\bigl[ \sum_{\tau=t}^T \gamma^{\tau-1} R_\tau \mid  h_t\bigr]$ denote the value function of policy $\vec \pi$ for the finite horizon model starting at history $h_t$ at time~$t$. Let $V^{\star}_{t,T}(h_t) \coloneqq \sup_{\vec \pi}V^{\vec \pi}_{t,T}(h_t)$ denote the optimal value function, where the optimization is over all history dependent policies. Moreover, let $\hat V_{t,T}(z_t)$ denote the optimal value function for the periodic MDP model constructed in \autoref{thm:convergence}. 
Let $\vec \pi_{\mu}$ denote the history-based policy defined in \autoref{sec:AIS}. 

Then, from~\cite[Theorem 9]{subramanian2022approximate} we have
\begin{equation}\label{eq:AIS-bound}
    \sup_{h_t} \bigl[ V^{\star}_{t,T}(h_t) - V^{\vec \pi_{\mu}}_{t,T}(h_t) \bigr] 
    \le 2 \sum_{\tau=t}^T \gamma^{\tau-t}\bigl[ \varepsilon_{\tau} + \gamma \delta_{\tau} \rho_{\F}(\hat V_{\tau+1,T}) \bigr]
\end{equation}
where we set $\hat V_{T+1,T}(z) \equiv 0$ for convenience. 

The following hold when we let $T \to \infty$.
\begin{itemize}[itemsep=0pt,topsep=0pt,partopsep=0pt]
    \item Since $R_t$ is uniformly bounded, $V^{\star}_{t,T}(h_t) \to V^{\star}_t(h_t)$ as $T \to \infty$.
    \item By the same argument, $V^{\vec \pi_{\mu}}_{t,T}(h_t) \to V^{\vec \pi_{\mu}}_t(h_t)$ as $T \to \infty$.
    \item By standard results for periodic MDPs (see \appref{app:periodic-MDP}), $\hat V_{t,T} \to V^{\MOD{t}}_{\mu}$ as $T \to \infty$.
    \item By definition, $\varepsilon_t \le \varepsilon_{t}^{\MOD{t}}$ and $\delta_t \le \delta_{t}^{\MOD{t}}$. 
\end{itemize}
Therefore, by taking $T \to \infty$ in~\eqref{eq:AIS-bound}, we get
\begin{equation}
    \sup_{h_t} \bigl[ V^{\star}_{t}(h_t) - V^{\vec \pi_{\mu}}_{t}(h_t) \bigr] 
    \le 2 \sum_{\tau=t}^{\infty} \gamma^{\tau-t}\bigl[ \varepsilon^{\MOD{\tau}}_{\tau} + \gamma \delta_{\tau}^{\MOD{\tau}} \rho_{\F}(\hat V^{\MOD{\tau+1}}) \bigr].
\end{equation}
The result then follows from observing that for $\tau \in \ALPHABET T(t,\ell)$,  $\epsilon_{t}^{\ell}$ and $\delta_{t}^{\ell}$ are non-decreasing sequences.

\section{Policy evaluation of an agent-state based policy}\label{app:evaluation}

The performance of any agent-state based policy can be evaluated via a slight generalization of ``cross-product MDP'' method originally presented in~\cite{Platzman1977}. This method has been rediscovered in slightly different forms multiple times~\cite{Littman1996,Cassandra1998,Hauskrecht1997,Hansen1998}. 

The key intuition is \autoref{lem:Markov}. Thus, for any agent-state based policy, $\{(S_t,Z_t)\}_{t \ge 1}$ is a Markov chain. The only difference in our setting is that the Markov chain is time-periodic. Thus, for any periodic agent-state based policy $(\pi^0, \dots, \pi^{L-1})$, we can identify the periodic rewards $(\bar r^0_, \dots, \bar r^{L-1})$ and periodic dynamics $(\bar P^0, \dots, \bar P^{L-1})$ (which depend on $\pi$ but we are not carrying that dependence in our notation) as follows:
\begin{align}
    \bar r^{\ell}(s,z) &= \sum_{a \in \ALPHABET A} \pi^{\ell}(a|z)r(s,a),  \\
    \bar P^{\ell}(s',z'|s,z) &= \sum_{(y,a) \in \ALPHABET Y \times \ALPHABET A} \pi^{\ell}(a|z) P(s',y'| s,a) 
    \IND_{\{z' = \phi(z,y',a)\}}.
\end{align}

We can then evaluate the performance of this time-periodic Markov chain via performance evaluation formulas for periodic MDPs (\autoref{app:periodic-MDP}). In particular, define
\begin{align}
    \tilde r &= \bar r^0 + \gamma \bar P^0 \bar r^1 + \dots + \gamma^{L-1} \bar P^0 \bar P^1 \cdots \bar P^{L-2} \bar r^{L-1}, \\
    \tilde P &= \bar P^0 \bar P^1 \cdots \bar P^{L-1},
\end{align}
to be the $L$-step cumulative rewards and dynamics for the time-periodic Markov chain. Then define
\begin{equation}
    \tilde V = (1 - \gamma^L \tilde P)^{-1} \tilde r
\end{equation}
Thus, $\tilde V(s,z)$ gives the performance of periodic policy $\pi$ when starting at initial state $(s,z)$. If the initial state is stochastic, we can average over the initial distribution.

\section{Reproducibility information} \label{app:reproducibility}

The hyperparameters for the numerical experiments presented in \autoref{sec:experiments} are shown in \autoref{table:hyperparameters}. The experiments were run on a computer cluster by running jobs that requested 
$2$-CPU nodes with $<8$GB memory. Each seed typically took less than $10$ minutes to execute.

\begin{table}[!ht]
\centering
\caption{Hyperparameters used in \autoref{ex:PASQL-example}}
\label{table:hyperparameters}
\begin{tabular}{@{}ll@{}} 
\toprule
Parameter & Value \\
\midrule
Training steps & $10^6$ \\ 
Start learn rate & $10^{-3}$  \\
End learn rate & $10^{-5}$ \\
Learn rate schedule & Exponential \\
Exponential decay rate & 1.0 \\
Number of random seeds & 25 \\
\bottomrule
\end{tabular}
\end{table}

\end{document}